\documentclass[11pt,letterpaper]{article}

\usepackage{fullpage}
\usepackage{mathpazo}
\usepackage[numbers, compress, sort]{natbib}
\usepackage[utf8]{inputenc} 
\usepackage[T1]{fontenc}    
\usepackage{hyperref}       
\usepackage{url}            
\usepackage{booktabs}       
\usepackage{amsfonts}       
\usepackage{nicefrac}       
\usepackage{microtype}      
\usepackage{caption}
\usepackage{amssymb}
\usepackage{graphicx}
\usepackage{threeparttable}
\usepackage{color}
\usepackage{complexity}
\usepackage{subcaption}
\usepackage{amsmath}
\usepackage{amsthm}
\usepackage{listings}
\usepackage{pifont}
\usepackage{tabularx}
\usepackage{fancyvrb}
\usepackage{comment}

\usepackage{float}
\usepackage{wrapfig}
\usepackage{mathtools}
\usepackage{thmtools,thm-restate}
\usepackage{paralist}
\usepackage{multirow}
\usepackage{scrextend}
\usepackage{enumitem}
\usepackage[noorphans]{quoting}
\usepackage{bm}
\usepackage{bbm}
\usepackage{thm-restate}
\usepackage{tabularx}
\usepackage{xcolor}
\usepackage{xspace}
\usepackage{times}
\usepackage[capitalize]{cleveref}
\usepackage{algorithm}
\usepackage{algpseudocode}

\theoremstyle{definition}

\newtheorem{theorem}{Theorem}[section]

\newtheorem*{remark}{Remark}

\newcommand{\Hquad}{\hspace{0.5em}} 
\newcommand{\red}[1]{{\color{red}{#1}}}

\renewcommand{\algorithmiccomment}[1]{\bgroup\hfill//~#1\egroup}

\newcommand{\defeq}{\vcentcolon=}

\definecolor{dkgreen}{rgb}{0,0.6,0}
\definecolor{gray}{rgb}{0.5,0.5,0.5}
\definecolor{mauve}{rgb}{0.58,0,0.82}

\newcommand{\etal}{\textit{et.al.}}

\newcommand{\eg}{\textit{e.g.}}
\newcommand{\etc}{\textit{etc.}}
\newcommand{\wrt}{\textit{w.r.t.}}

\title{\Large \textbf{APRIL: Finding the Achilles' Heel on Privacy for Vision Transformers}}
\date{}
\author{Jiahao Lu$^{1,2}$, Xi Sheryl Zhang$^{1}$,  Tianli Zhao$^{1,2}$, Xiangyu He$^{1,2}$, Jian Cheng$^1$\\
	$^1$\ Institute of Automation, Chinese Academy of Sciences \\
	$^2$\ School of Artificial Intelligence, University of Chinese Academy of Sciences
	\\
	\texttt{\{lujiahao2019, xi.zhang, zhaotianli2019\}@ia.ac.cn},\\ \texttt{\{xiangyu.he, jcheng\}@nlpr.ia.ac.cn}
}

\begin{document}
	
	\maketitle
	
	\begin{abstract}
		Federated learning frameworks typically require collaborators to share their local gradient updates of a common model instead of sharing training data to preserve privacy. However, prior works on Gradient Leakage Attacks showed that private training data can be revealed from gradients. So far almost all relevant works base their attacks on fully-connected or convolutional neural networks. Given the recent overwhelmingly rising trend of adapting Transformers to solve multifarious vision tasks, it is highly valuable to investigate the privacy risk of vision transformers. In this paper, we analyse the gradient leakage risk of self-attention based mechanism in both theoretical and practical manners. Particularly, we propose \textbf{APRIL} - \textbf{A}ttention \textbf{PRI}vacy \textbf{L}eakage, which poses a strong threat to self-attention inspired models such as ViT. Showing how vision Transformers are at the risk of privacy leakage via gradients, we urge the significance of designing privacy-safer Transformer models and defending schemes.
	\end{abstract}
	
	\section{Introduction}
	\label{sec:intro}
	
	Federated or collaborative learning~\cite{17federated} have been gaining massive attention from both academia~\cite{li2014scaling,konevcny2016federated} and industry~\cite{The_TensorFlow_Federated_Authors_TensorFlow_Federated_2018,fed19}. For the purpose of privacy-preserving, typical federated learning keeps local training data private and trains a global model by sharing its gradients collaboratively. By avoiding to transmit the raw data directly to a central server, the learning paradigm is widely believed to offer sufficient privacy. Thereby, it has been employed in real-world applications, especially when user privacy is highly sensitive, \eg hospital data~\cite{brisimi2018federated, jochems2016distributed}.

	Whilst this setting prevents direct privacy leakage by keeping training data invisible to collaborators, a recent line of the works \cite{dlg, idlg, ig, gradinversion, r-gap, tag} demonstrates that it is possible to (partially) recover private training data from the model gradients. This attack dubbed \emph{gradient leakage} or \emph {gradient inversion} poses a severe threat to the federated learning systems. The previous works primarily focus on inverting gradients from fully connected networks (FCNs) or convolutional neural networks (CNNs). Particularly, Yin \etal \cite{gradinversion} recover images with high fidelity relying on gradient matching with BatchNorm layer statistics; Zhu \etal \cite{r-gap} theoretically analyse the risk of certain architectures to enable the full recovery. One intriguing question of our interest is that, \textit{does gradient privacy leakage occur in the context of architectures other than FCNs and CNNs?} 
	
	The recent years have witnessed a surge of methods about Transformer~\cite{transformer}. As an inherently different architecture, Transformer can build large scale contextual representation models, and achieve impressive results in a broad set of natural language tasks. For instance, huge pre-trained language models including BERT~\cite{devlin-etal-2019-bert}, XLNet~\cite{yang2019xlnet}, GPT-3~\cite{brown2020language}, Megatron-LM~\cite{shoeybi2019megatron}, and so forth are established on the basis of Transformers. 
	Inspired by the success, original works~\cite{Wang_2018_CVPR, ramachandran2019stand,cordonnier2019relationship,bello2019attention} seek to the feasibility of leveraging self-attention mechanism with convolutional layers to vision tasks.
	Then, DETR~\cite{detr} makes pioneering progress to use Transformer in object detection and ViT~\cite{vit} resoundingly succeeds in image classification with a pure Transformer architecture. Coming after ViT, dozens of works manage to integrate Transformer into various computer vision tasks ~\cite{swin, li2021mst, yanlearning, liu2021paint, yang2021transpose, yu2021pointr, xie2021segformer, fang2021you}. Notably, vision Transformers are known to be extremely data-hungry~\cite{vit}, which makes the large-scale learning in the federated fashion more favorable.

	Despite the rapid progress aforementioned, there is a high chance that vision Transformers suffer the gradient leakage risk. Nevertheless, the line of the study on this privacy issue is absent. Although the prior work~\cite{tag} provides an attack algorithm to recover private training data for a Transformer-based language model via an optimization process, the inherent reason of Transformer's vulnerability is unclear. Different with leakage on Transformer in natural language tasks~\cite{tag}, we claim that vision Transformers with the position embedding not only encodes positional information for patches but also enables gradient inversion from the layer. In this paper, we introduce a novel analytic gradient leakage to reveal why vision Transformers are easy to be attacked. Furthermore, we explore gradient leakage by recovery mechanisms based on an optimization approach and provide a new insight about the position embedding. Our results of gradient attack will shed light on future designs for privacy-preserving vision Transformers.           

	To summarize, our contributions are as follows:
	\begin{itemize}
		\item{We prove that for the classic self-attention module, the input data can be perfectly reconstructed without solving an intractable optimization problem, if the gradient \wrt the input is known.}
		
		\item{We demonstrate that jointly using self-attention and learnable position embedding place the model at severe privacy risk. The attacker obtain a closed-form solution to the privacy leakage under certain conditions, regardless of the complexity of networks.}
		
		\item{We propose an Attention Privacy Leakage (APRIL) attack, to discover the Archilles' Heel. As an alternative, APRIL performs an optimization-based attack, apart from the closed-form attack. The attacks show that our results superior to SOTA.}
		
		
		\item{We suggest to switch the learnable position embedding to a fixed one as the defense against privacy attacks. Empirical results certify the effectiveness of our defending scheme. 
			
		}
	\end{itemize}

	\section{Preliminary}
	\label{sec:relatedwork}
	\noindent
	\textbf{Federated Learning.}
	Federated learning~\cite{17federated} offers the scheme that trains statistical models collaboratively involving multiple data owners. Due to the advances containing privacy, large-scale training, and distributed optimization, federated learning methods have been deployed by applications which require computing at the edge~\cite{qi2020privacy,hard2018federated,huang2019patient,brisimi2018federated}. In this scenario, we aim to learn a global model by locally processed \textit{client} data and communicating intermediate updates among each other. Formally, the typical goal is minimizing the following loss function $l$ with parameters $w$,
	\begin{equation}
	\min_{w} l_{w} (x, y),  \Hquad \textnormal{where}\Hquad l_{w} (x, y) \defeq \sum_{i=1}^{N}p_i l_{w}^i (x_i, y_i)
	\end{equation}
	where $p_i\geq 0$ and $\sum_i p_i=1$. Since the $N$ clients owns the private training data. Let $(x_i, y_i)$ denote samples available locally for the $i$th client, and $l_{w}^i (x_i, y_i)$ denote the local loss function. In order to preserve data privacy, clients periodically upload their gradients $\nabla_{w} l_{w}^i (x_i, y_i)$ computed on their own local batch. The server aggregates gradients from all clients, updates the model using gradient descent and then sends back the updated parameters to every client.
	
	\noindent
	\textbf{Gradient Leakage Attack.}
	As an \textit{honest-but-curious} adversary at the server side may reconstruct clients' private training data without mess up the training process, sharing gradients in federated learning is no longer safe for client data. Albeit recovering image using gradients comes up with the attackers under weak constraints, endeavors of the existing threat models mainly focus on two directions: \textit{optimization-based attacks} and \textit{closed-form attacks}. 
	
	The basic recovery mechanism is defined by optimizing an euclidean distance as follows,  
	\begin{equation}
	\min_{x'_i,y'_i} \|\nabla_{w} l_{w}^i (x_i, y_i)-\nabla_{w} l^i_{w} (x'_i, y'_i)\|^2
	\label{eq:gradmatching}
	\end{equation}
	Deep leakage~\cite{dlg} minimizes the matching term of gradients from dummy input $(x'_i, y'_i)$ and those from real input $(x_i, y_i)$\footnote{We omit the index $i$ for clients in followings to manifest that the algorithm can work for any client.}. 
	On the top of this proposal, iDLG~\cite{idlg} finds that in fact we can derive the ground-truth label from the gradient of the last fully connected layer. By eliminating one optimization objective in Eq.(\ref{eq:gradmatching}), the attack procedure becomes even faster and smoother. Also, Geiping \etal~\cite{ig} prove that inversion from gradient is strictly less difficult than recovery from visual representations. 
	GradInversion~\cite{gradinversion} incorporates heuristic image prior as regularization by utilizing BatchNorm matching loss and group consistency loss for image fidelity. 
	Lately, GIML~\cite{genprior} illustrates that a generative model pre-trained on data distribution can be exploited for reconstruction.  
	
	One essential challenge of optimization procedures is that there is no sufficient condition for the uniqueness of the optimizer. The closed-form attack, as another of the ingredients in this line, is introduced by Phong \etal~\cite{phong2018privacy}, which reconstructs inputs using a shallow network such as a single-layer perceptron. R-GAP~\cite{r-gap} is the first derivation-based approach to perform an attack on CNNs, which models the problem as linear systems with closed-form solutions. Compared to the optimization-based method, analytic gradient leakage heavily depends on the architecture of neural networks and thus cannot always guarantee a solution.    
	
	\noindent
	\textbf{Transformers.}
	Transformer~\cite{transformer} is introduced for neural machine translation to model the long-term correlation between tokens meanwhile represent dependencies between any two distant tokens. The key of outstanding representative capability comes from stacking multi-head self-attention modules. 
	Recently, vision Transformers and its variants are broadly used for powerful backbones~\cite{vit, deit, swin}, object detection~\cite{detr}, semantic segmentation~\cite{sstrans}, image generation~\cite{parmar2018image, jiang2021transgan, chen2020generative}, \etc.
	
	Given the fundamentals of vision Transformer, we will investigate the gradient leakage in terms of closed-formed and optimization-based manners. Thus far, almost all the gradient leakage attacks adopt CNNs as the testing ground, typically using VGG or ResNet. Besides, TAG~\cite{tag} conducts experiments on popular language models using Transformers without concerning any analytic solution as well as the function of position embedding. 
	

	\section{APRIL: Attention PRIvacy Leakage}
	In light of above missing for vision transformer, we first prove that gradient attacks on self-attention can be analytically conducted. Next, we will discuss the possible leakage from the position embedding based on its analytic solution, which naturally gives rise to two attack approaches.
	
	\subsection{Analytic Gradient Attack on Self-Attention}
	It has been proved that the closed-form solution for input $x$ can always be perfectly obtained on a fully-connected layer $\sigma(Wx+b)=z$, through deriving gradients \wrt weight $W$ and bias $b$. The non-linear function $\sigma$ is an activation~\cite{phong2018privacy}. 
	In this work, we delve into a more subtle formulation of a self-attention to demonstrate the existence of the closed-form solution.
	
	\begin{theorem}(Input Recovery).
		\label{theo:1}
		Assume a self-attention module expressed as:
		\begin{align}
		&Qz = q; Kz = k; Vz = v;\label{eq:qkv}\\ 
		&\frac{softmax(q\cdot k^{T})}{\sqrt{d_k}} \cdot v = h\label{eq:attn}\\ 
		&Wh = a; \label{eq:linear}
		\end{align}
		where $z$ is the input of the self-attention module, $a$ is the output of the module. Let $Q, K, V, W$ denote the weight matrix of query, key, value and projection, and $q, k, v, h$ denote the intermediate feature map. Suppose the loss function can be written as 
		\begin{equation*}
		l = l(f(a),y)
		\end{equation*}
		If the derivative of loss $l$ \wrt the input $z$ is known, then the input can be recovered uniquely from the network's gradients by solving the following linear system:
		
		\begin{equation*}
		\frac{\partial l}{\partial z}z^T = Q^T\frac{\partial l}{\partial Q} + K^T\frac{\partial l}{\partial K} + V^T\frac{\partial l}{\partial V}
		\end{equation*}
		
	\end{theorem}
	
	
	\begin{proof}
		\rm{
			In spite of the non-linear formulation of self-attention modules, the gradients \wrt $z$ can be derived in a succinct linear equation:
			\begin{equation}
			\frac{\partial l}{\partial z}  = Q^T \frac{\partial l}{\partial q}  + K^T\frac{\partial l}{\partial k}  + V^T\frac{\partial l}{\partial v}   
			\label{eq:dx}
			\end{equation} 
			
			Again, according to the chain rule of derivatives, we can derive the gradients \wrt $Q$, $K$ and $V$ from \cref{eq:qkv}:
			\begin{equation}
			\frac{\partial l}{\partial Q} = \frac{\partial l}{\partial q}z^T  \quad
			\frac{\partial l}{\partial K} = \frac{\partial l}{\partial k}z^T  \quad
			\frac{\partial l}{\partial V} = \frac{\partial l}{\partial v}z^T  
			\label{eq:dQKV}
			\end{equation}
			
			By multiplying $z^T$ to both sides of \cref{eq:dx} and substituting \cref{eq:dQKV}, we obtain:
			\begin{equation}
			\begin{aligned}
			\frac{\partial l}{\partial z}z^T &= Q^T \frac{\partial l}{\partial q}z^T  + K^T\frac{\partial l}{\partial k} z^T + V^T\frac{\partial l}{\partial v}z^T \\
			&= Q^T\frac{\partial l}{\partial Q} + K^T\frac{\partial l}{\partial K} + V^T\frac{\partial l}{\partial V} \label{eq:finding}
			\end{aligned} 
			\end{equation}
			which completes the proof.
		}
	\end{proof}
	\begin{remark}
		Surprisingly we find that for a malicious attacker aiming to recover the input data $z$.  
		Since an adversary in the context of federated learning knows both learnable parameters and gradients \wrt them, in this case, $Q$, $K$, $V$ and $\frac{\partial l}{\partial Q}$, $\frac{\partial l}{\partial K}$, $\frac{\partial l}{\partial V}$. The right side of \cref{eq:finding} is known. As a result, once the derivative of the loss \wrt the input $\frac{\partial l}{\partial z}$ is exposed to the adversary, the attacker can easily get an accurate reconstruction of $z$ by solving the linear equation system in \cref{eq:finding}. 
	\end{remark}
	
	\begin{algorithm}[t!]
		\centering
		\caption{Closed-Form APRIL}
		\label{alg:closed-form}

		\begin{algorithmic}[1]
			\Require {Attention module: $F(z, w)$; 
				Module weights $w$;
				Module gradients $\frac{\partial l}{\partial w}$ 
				Derivative of loss \wrt $z$: $\frac{\partial l}{\partial z}$}
			\Ensure {Embedding feed into attention module: $z$}  
			\Procedure{APRIL-closed-form}{$F, w, \frac{\partial l}{\partial w}, \frac{\partial l}{\partial z}$}
			\State Extract $Q, K, V$ from module weights $w$
			\State Extract $\frac{\partial l}{\partial Q}, \frac{\partial l}{\partial K}, \frac{\partial l}{\partial V}$ from module gradients $\frac{\partial l}{\partial w}$
			\State $A \gets \frac{\partial l}{\partial z}$
			\State $b \gets Q^T\cdot \frac{\partial l}{\partial Q} + V^T \cdot \frac{\partial l}{\partial V} + K^T \cdot \frac{\partial l}{\partial K}$ 
			\State $z \gets A^{\dag}\cdot b$ \Comment{ $A^{\dag}$: Moore-Penrose pseudoinverse of $A$ \Hquad\quad\quad}
			\State $z \gets z^T $ \Comment{Transpose\quad}
			
			\EndProcedure
		\end{algorithmic}
	\end{algorithm}
	
	\noindent
	\textbf{Solution Feasibility.} Suppose the dimension of the embedding $z$ is $\mathbb{R}^{p \times c}$, with patch number $p$ and channel number $c$. This linear system has $p\times c$ unknown variables yet $c \times c$ linear constraints. Since deep neural networks normally have wide channels for the sake of expressiveness, $c \gg p$ in the most of model designs, which leads to an overdetermined problem and thereby a solvable result. 
	In other words, $z$ can be accurately reconstructed if $\frac{\partial l}{\partial z}$ is available. The entire procedure of the closed-form attack is presented in Alg.\ref{alg:closed-form}.
	
	
	

	\begin{figure}[t]
		\centering
		\includegraphics[width=0.8\linewidth]{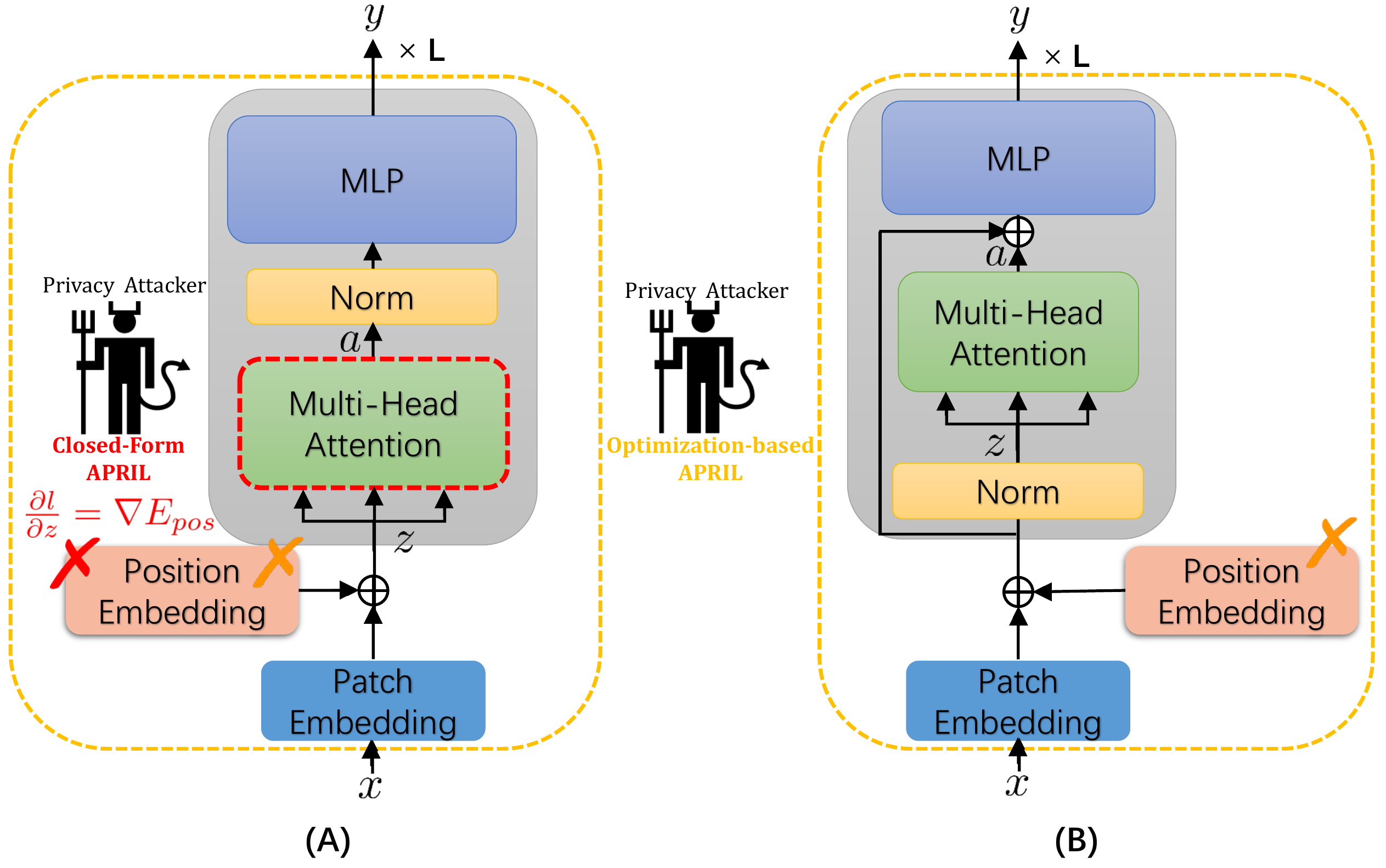}

		\caption{We consider two Transformer designs throughout the paper. (A): Encoder modules stack multi-head attention, normalization, and MLP in VGG-style. (B): A real-world design as introduced in ViT~\cite{vit}. The architecture in (A) satisfies the precondition of a closed-form APRIL attack, since the output of position embedding is exactly input for multi-head attention, showing by the red dashed line box. In contrast, the optimization-based APRIL attack can be placed in any design of architectures, showing by the yellow dashed line boxes in (A) and (B).}
		\label{fig:arch}
		\vspace{0.5cm}
	\end{figure}
	
	
	\subsection{Position Embedding: The Achilles' Heel}
	Now we focus on the how to access the critical derivative $\frac{\partial l}{\partial z}$ by introducing the leakage caused by the position embedding. Under general settings of federated learning, the sensitive information related with $z$ is invisible from users' side. Here, we show that $\frac{\partial l}{\partial z}$ is unfortunately exposed by gradient sharing for vision Transformers with a learnable position embedding. Specifically, we give the following theorem to illustrate the leakage.      
	
	
	
	\begin{theorem} (Gradient Leakage).
		\label{theo:2}
		For a Transformer with learnable position embedding $E_{pos}$, the derivative of loss \wrt $E_{pos}$ can be given by 
		\begin{equation}
		\frac{\partial l}{\partial E_{pos}}=\frac{\partial l}{\partial z}
		\label{eq:posder}
		\end{equation}
		where $\frac{\partial l}{\partial z}$ is defined by the linear system in \cref{theo:1}. 
	\end{theorem}
	
	\begin{proof}\rm{
			Without loss of generality, the embedding $z$ defined by \cref{theo:1} can be divided into a patch embedding $E_{patch}$ and a learnable position embedding $E_{pos}$ as,   
			\begin{equation}
			z = E_{patch} + E_{pos}
			\label{eq:posemb}
			\end{equation}
			Straightforwardly, we compute the derivative of loss \wrt $E_{pos}$ using \cref{eq:posemb}, \cref{eq:posder} holds.
	}\end{proof}
	\begin{remark}
		The sensitive information $\frac{\partial l}{\partial z}$ is exactly the same as the gradient of the position embedding $\frac{\partial l}{\partial E_{pos}}$, denoting as $\nabla E_{pos}$ for simplicity. As model gradients are sharing, $\nabla E_{pos}$ is available for not only legal users but also potential adversaries, which means a successful attack on self-attention inputs.        
	\end{remark}
	
	While vision Transformers~\cite{vit, swin, wu2021rethinking} embody prominent accuracy raise using learnable position embeddings rather than the fixed ones, updating of parameter $E_{pos}$ will result in privacy-preserving troubles based on our theory. More severely, the attacker only requires a learnable position embedding and a self-attention stacked at the bottom in VGG-style, regardless of the complexity of the rest architecture, as shown in \cref{fig:arch} (A). At a colloquial level, we suggest two strategies to alleviate this leakage, which is either employing one fixed position embedding instead of the learnable one or updating $\nabla E_{pos}$ only on local clients without transmission.                                               
	
	
	
	
	
	\begin{algorithm*}[t!]
		\centering
		\caption{Optimization-based APRIL}
		\label{alg:opt-based}

		\begin{algorithmic}[1]
			\Require{Transformer with learnable position embedding: $F(x, w)$;
				Module parameter weights : $w$;
				Module parameter gradients: $\nabla w $;
				APRIL loss term scaler: $\alpha$}
			\Ensure{Image feed into the self-attention module: $x$}  
			\Procedure{APRIL-optimization-attack}{$F, w, \nabla w$}
			\State Extract final linear layer weights $w_{fc}$ from $w$
			\State $y \gets i$ \quad \textnormal{s.t.} $\nabla {w_{fc}^{i}}^T  \nabla w_{fc}^{j} \le 0, \forall j\neq i $ \Comment{ Extract ground-truth label using the iDLG trick\quad\quad\quad}
			\State Extract position embedding layer's gradients $\nabla E_{pos}$ from $\nabla w$
			\State $x' \gets \mathcal{N}(0,1)$ \Comment{Initialize the dummy input\quad}
			
			\State \textbf{While}\quad\textnormal{not converged}\quad\textbf{do}
			
			\State \quad $\frac{\partial l}{\partial w} ' \gets \partial l(F(x'; w), y) / \partial w$ \Comment{Calculate dummy gradients\quad}
			
			\State \quad $\mathcal{L}_G = \| \nabla w'  - \nabla w \|^2_F$ \Comment{Calculate L-2 difference between gradients\quad}
			
			\State \quad $\frac{\partial l}{\partial E_{pos}'} \gets \partial l(F(x'; w), y) / \partial E_{pos}'$ \Comment{Calculate the derivative of dummy loss \wrt dummy input\quad}
			
			\State \quad $\mathcal{L}_{A} = - \frac{<\nabla E_{pos}, \nabla E_{pos}^{'}>}{\| \nabla E_{pos}\| \cdot \|\nabla E_{pos}^{'}\|}$	 \Comment{Calculate cosine distance between derivative of input\quad}
			
			\State \quad $\mathcal{L} = \mathcal{L}_G + \alpha \mathcal{L}_A$
			
			\State \quad $x' \gets x' - \eta \nabla_{x'} \mathcal{L}$ \Comment{Update the dummy input\quad\quad\quad}
			\EndProcedure
		\end{algorithmic}
	\end{algorithm*}

	\subsection{APRIL Attacks on Vision Transformer}
	So far the analytic gradient attack have succeeded in reconstructing input embedding $z$ meanwhile obtaining the gradient of position embedding $\nabla E_{pos}$. One question is that can APRIL take advantage of the sensitive information to further recover the original input $x$. The answer is affirmative.     
	
	\noindent
	\textbf{Closed-Form APRIL.} 
	As a matter of the fact, APRIL attacker can inverse the embedding via a linear projection to get original input pixels. For a vision Transformer, the input image is partitioned into many patches and sent through a so-called ``Patch Embedding'' layer, defined as
	\begin{equation}
	E_{patch} = W_{p} x 
	\end{equation}
	The bias term is omitted since it can be represented in an augmented matrix $W_{p}$. With $W_{p}$, pixels are linearly mapped to features, and the attacker calculates the original pixels by left-multiply its pseudo-inverse.
	
	\noindent
	\textbf{Optimization-based APRIL.} 
	Given the linear system in \cref{theo:1}, it can also be decomposed into two components as $z$ and $\nabla E_{pos}$ based on Eq.(\ref{eq:posder}). Arguably, component $\nabla E_{pos}$ indicates the directions of the gradients of position embeddings and contributes to the linear system indepentently with data. Considering the significance of the learnable position embedding in gradient leakage, intuitively, matching the updating direction of $E_{pos}$ with an direction caused by dummy data can do benefits on the recovery. Therefore, we proposed an optimization-based attacker with constraints on the distance of $\nabla E_{pos}$. 
	To do so, apart from architecture in \cref{fig:arch}(A), typical design of ViT illustrated in \cref{fig:arch}(B) using normalization and residual connections with a different stacked order can also be attacked by our proposed APRIL. 

	For expression simicity, we use $\nabla w'$ and $\nabla w$ denote the gradients of parameter collections for dummy data and real inputs, respectively. In detail, the new integrated term of gradients of $\nabla E_{pos}$ is set as $ \mathcal{L}_A$. For modelling directional information, we utilize a cosine similarity between real and dummy position embedding derivatives as a regularization. The intact optimization problem is written as
	\begin{equation}
	\begin{aligned}
	\mathcal{L} & = \mathcal{L}_G + \alpha \mathcal{L}_A \\ & = \| \nabla w' - \nabla w \|^2_F - \alpha\cdot\frac{<\nabla E_{pos}, \nabla E_{pos}^{'}>}{\| \nabla E_{pos}\| \cdot \|\nabla E_{pos}^{'}\|}.
	\label{eq:opt-based}
	\end{aligned}
	\end{equation}
	where hyperparameter $\alpha$ balances the contributions of two matching losses. Eventually, we set Eq.(\ref{eq:opt-based}) is another variant of our proposed method, optimization-based APRIL attack. The associated algorithm is described in Alg.\ref{alg:opt-based}. By enforcing a gradient matching on the learnable position embedding, it is plaguily easy to break privacy in a vision Transformer.

	\section{Experiments}
	\label{sec:experiments}
	\noindent
	In this section, we aspire to carry out experiments to answer the following questions: (1) To what extend can APRIL break privacy of a Transformer? (2) How strong is the APRIL attack compared to existing privacy attack methods? (3) What defensive strategy can we take to alleviate APRIL attack? (4) How to testify the functionality of position embedding in privacy preserving?
	
	We mainly carry out experiments in the setting of image classification; however, APRIL as a universal attack for Transformers can also be performed in a language task setting.
	\begin{figure*}[t]
		\centering
		\vspace{-4mm}
		\includegraphics[width=0.7\linewidth]{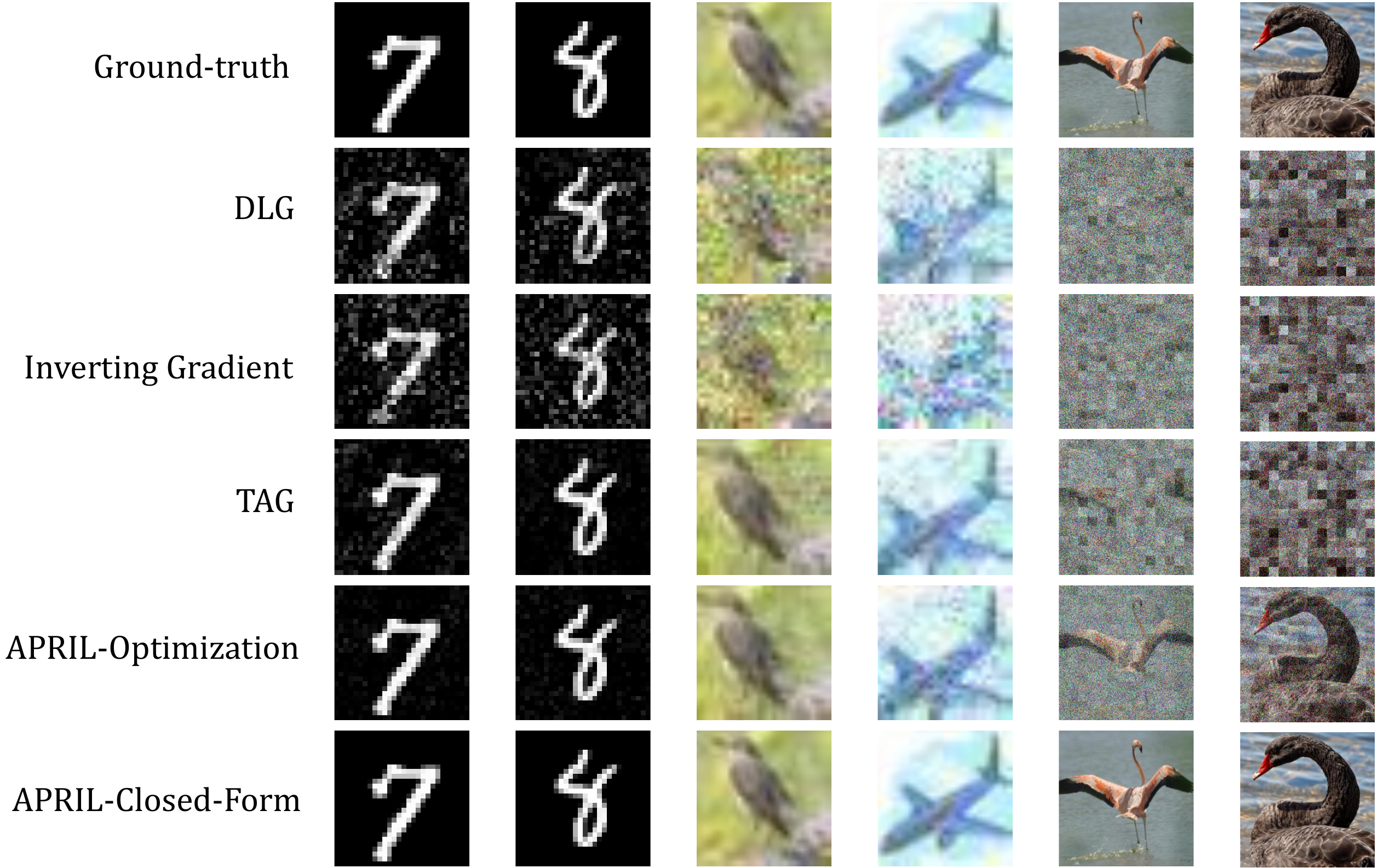}

		\caption{Results for different privacy attacking approaches on Architecture (A). For optimization-based attacks, we use an Adam optimizer to update 800 iterations for MNIST, 1500 iterations for CIFAR-10 and 5000 iterations for ImageNet. Please zoom-in to see details. We put more results in Appendix.}
		\label{fig:rescomp}
	\end{figure*}
	
	\begin{table*}[t]
		\centering
		
		\label{tab:MSE}
		\scalebox{0.8}{
			\begin{tabular}{|c|c|c|c|c|c|c|c|}
				\hline
				\multicolumn{2}{|c|}{\multirow{2}{*}{Attack}}  & \multicolumn{2}{c|}{MNIST}&\multicolumn{2}{c|}{CIFAR-10}&\multicolumn{2}{c|}{ImageNet}\\
				
				\cline{3-8}
				
				\multicolumn{2}{|c|}{} &MSE&SSIM & MSE & SSIM & MSE & SSIM  \\
				\hline
				
				\multicolumn{2}{|l|}{DLG~\cite{dlg}} &1.291e-04 $\pm$ 2.954e-04	&0.997 $\pm$ 0.003	& 0.017 $\pm$ 0.009	&0.959$\pm$0.045	&1.328$\pm$0.593 & 0.056 $\pm$ 0.027\\ 
				\hline		
				
				\multicolumn{2}{|l|}{IG~\cite{ig}} &0.043$\pm$0.022	&0.833$\pm$0.076	&0.125$\pm$0.102	&0.635$\pm$0.165	&1.671$\pm$0.653 & 0.029$\pm$0.013\\ 
				\hline		
				
				\multicolumn{2}{|l|}{TAG~\cite{tag}} &\textbf{3.438e-05 $\pm$ 1.322e-05}	&\textbf{0.998$\pm$0.002}	&0.006 $\pm$ 0.005	&0.965$\pm$0.047	&1.180 $\pm$ 0.473 & 0.062 $\pm$ 0.026\\ 
				\hline
				
				\multicolumn{2}{|l|}{APRIL} &4.796e-05$\pm$3.593e-05  	&\textbf{0.998 $\pm$ 0.002}	&\textbf{0.002$\pm$0.006}	&\textbf{0.991 $\pm$ 0.027}	&\textbf{1.092$\pm$0.663} & \textbf{0.099 $\pm$ 0.046}\\ 
				\hline
				
			\end{tabular}
		}
		\caption{ Mean and standard deviation for MSE of 500 reconstructions on MNIST, CIFAR-10 and ImageNet validation datasets, respectively. We randomly selected 50 images from each class in MNIST and CIFAR-10, and one image for random 500 classes in ImageNet.}
	\end{table*}
	\begin{figure*}[h]
		\centering
		\includegraphics[width=\linewidth]{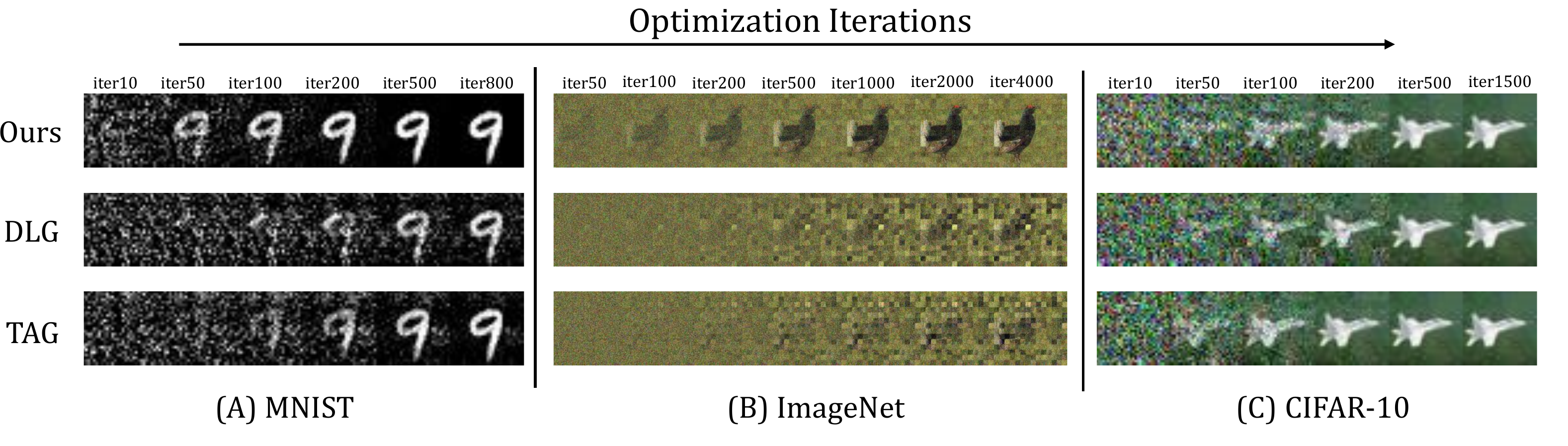}
		\vspace{-0.5cm}
		\caption{Visualization of the optimization process for optimization-based APRIL, DLG and TAG. Our approach has faster convergence speed and does not easily fall into bad local minima, thus yields a prominently better reconstruction result.}
		\setlength{\belowcaptionskip}{-8cm}
		\label{fig:evolution}
	\end{figure*}
	
	
	We carry out experiments on two different architectures, as illustrated in \cref{fig:arch}, architecture (A) has a position embedding layer directly connected to attention module, making it possible to perform APRIL-closed-form attack. Architecture (B) has the same structure as ViT-Base ~\cite{vit}, which is composed of multiple encoders, each with a normalization layer before attention module as well as a residual connection. For small datasets like CIFAR and MNIST, we refer to the implementation of ViT-CIFAR\footnote{https://github.com/omihub777/ViT-CIFAR}. We set the hidden dimension to 384, attention head to 4, and partition input images into 4 patches. The encoder depth is 4, after that the classification token is connected to a classification head. For experiments on ImageNet, we follow the original ViT design\footnote{https://github.com/lucidrains/vit-pytorch} , which includes 16x16 image patch size, 12 attention heads with hidden dimensions of 768.
	
	\subsection{APRIL as the Gradient Attack}
	We first apply APRIL attacks on Architecture (A) and compare it with other attacking approaches. As \cref{fig:rescomp} shows, closed-form APRIL attack provides a perfect reconstruction, which shows nearly no difference to the original input, which proves the correctness of our theorem.
	Comparing optimization-based attacks, for easy tasks like MNIST and CIFAR with a clean background, all existing attacking algorithms show their ability to break privacy, although DLG~\cite{dlg} and IG~\cite{ig} have some noises in their results. 
	The comparison is obvious for ImageNet reconstructions, where DLG, IG and TAG reconstructions are nearly unrecognizable to humans, with strong block artifacts. In contrast, the proposed APRIL-Optimization attack behaves prominently better, which reveals quite a lot of sensitive information from the source image, including details like the color and shape of the content.
	
	We further study the optimization procedure of reconstruction, shown in \cref{fig:evolution}. We illustrate the updating process of the dummy image. We can observe that all three approaches can break some sort of privacy, but they differ in convergence speed and final effects.  An apparent observation is that our optimization-based APRIL converges consistently faster than the other two. Besides, our approach generally ends up at a better terminal point, which results in smoother and cleaner image reconstructions.
	
	Apart from visualization results, we want to have a quantitative comparison between these optimization-based attacks. We carry out this experiment on Architecture(B), where we do not have the condition to use closed-form APRIL attack.  The statistical results from \cref{tab:MSE} shows  consistent good performance of APRIL, and we obtain best results nearly across every task setting. 
	
	Finally, we try to attack batched input images. As shown in \cref{fig:batch_opt}, our optimization results on batched input achieved impressive results as well. Note here we used the trick introduced by Yin \etal ~\cite{gradinversion} to restore batch labels before optimization. More results are put in Appendix. 
	\begin{figure}
		\centering	
		\includegraphics[width=\linewidth]{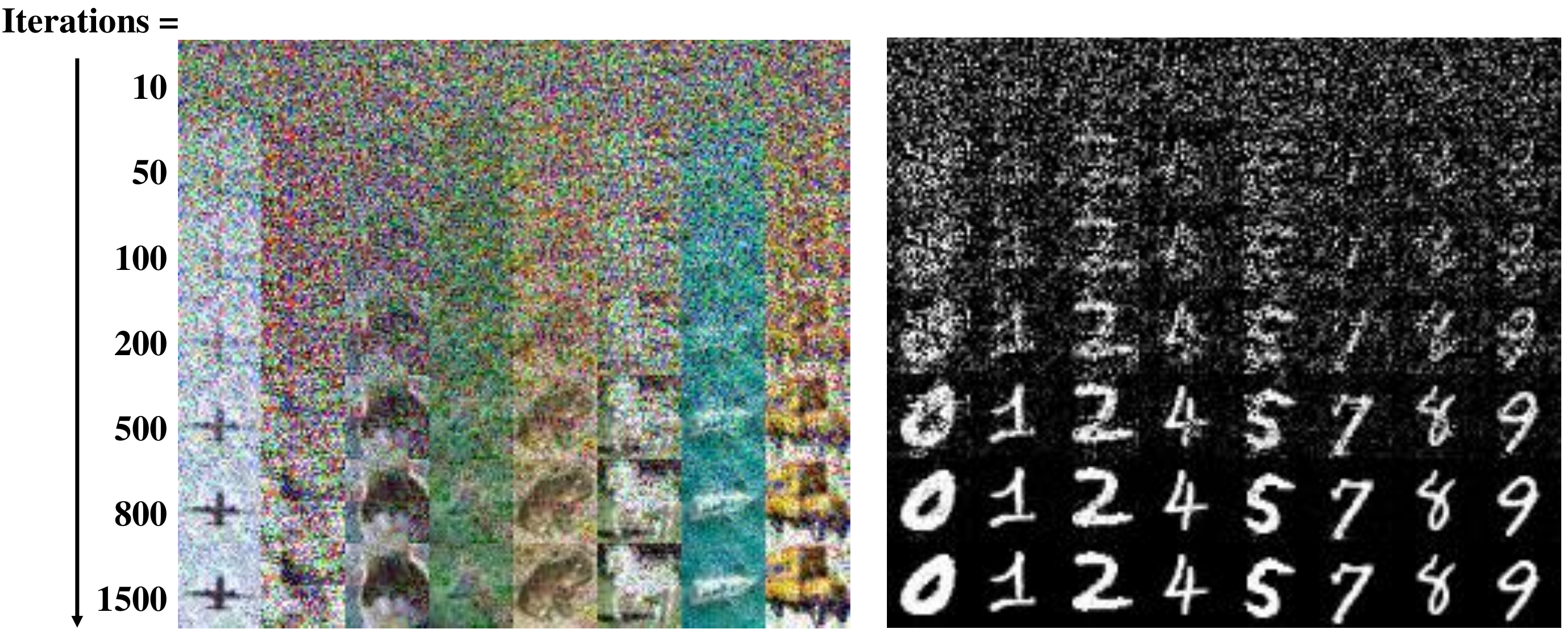}
		\caption{Optimization-based APRIL attack on batched inputs.}
		
		\label{fig:batch_opt}
	\end{figure}
	It's worth mentioning that the use of a closed-form APRIL attack is limited under batched setting, since the gradients are contributed by all samples in a batch, and we can only solve an "averaged" version of $z$ in \cref{eq:finding}. We give more reconstruction results and discuss more thoroughly on the phenomenon in Appendix.

	All experiments shown above demonstrate that the proposed APRIL outperforms all existing privacy attack approaches in the context of Transformer, thus posing a strong threat to Vision Transformers. 
	
	\subsection{APRIL-Inspired Defense Strategy}
	\label{subsec:nopos}
	\noindent
	\textbf{How robust is the closed-form APRIL?} 
	In the last subsection, we show that under certain conditions, closed-form APRIL attack can be executed to get almost perfect reconstructions. The execution of this attack is based on solving a linear system. Linear systems can be unstable and ill-conditioned when the condition number is large. With this knowledge, we are interested to know how much disturbance can APRIL bear to remain a good attack? We discuss a few defensive strategies towards APRIL. 
	
	\begin{figure}
		\centering
		\begin{minipage}[t]{0.22\linewidth}
			\centering
			\includegraphics[width=\linewidth]{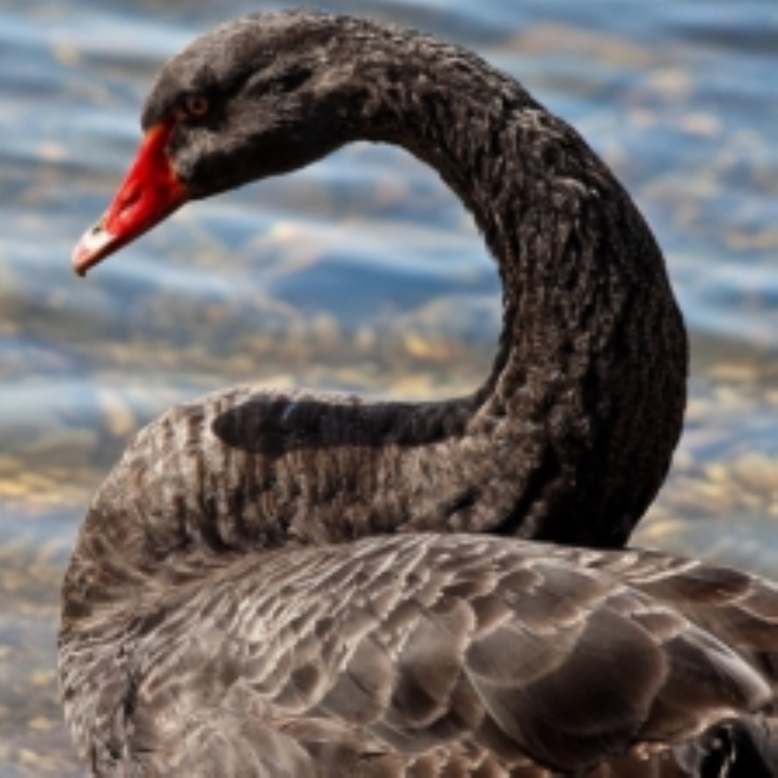}
			\vspace{-8mm}
			\caption*{\footnotesize hidden dimension=768}
		\end{minipage}	
		\hfill
		\begin{minipage}[t]{0.22\linewidth}
			\centering
			\includegraphics[width=\linewidth]{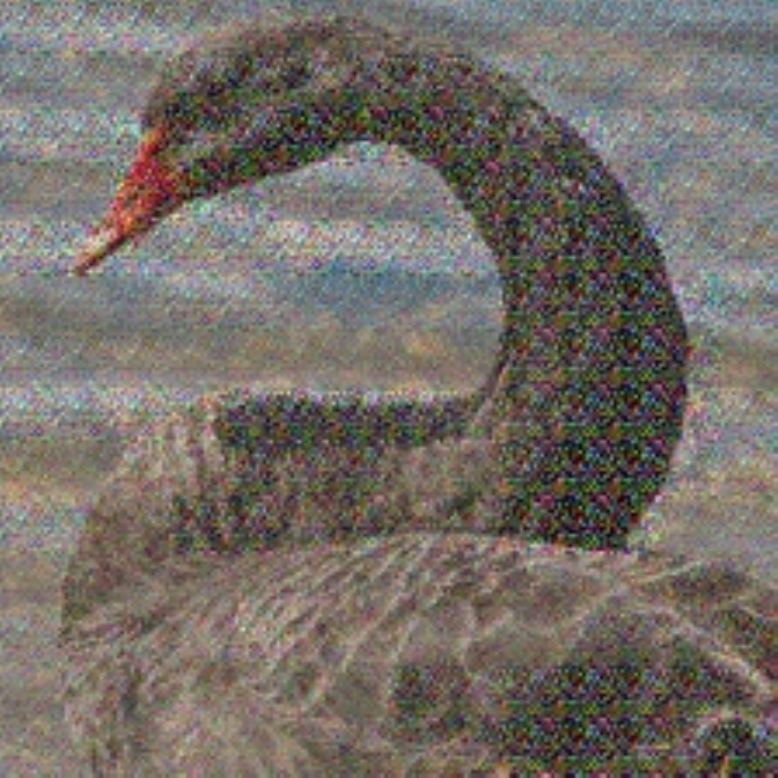}
			\vspace{-8mm}
			\caption*{\footnotesize hidden dimension=384}
		\end{minipage}
		\hfill
		\begin{minipage}[t]{0.22\linewidth}
			\centering
			\includegraphics[width=\linewidth]{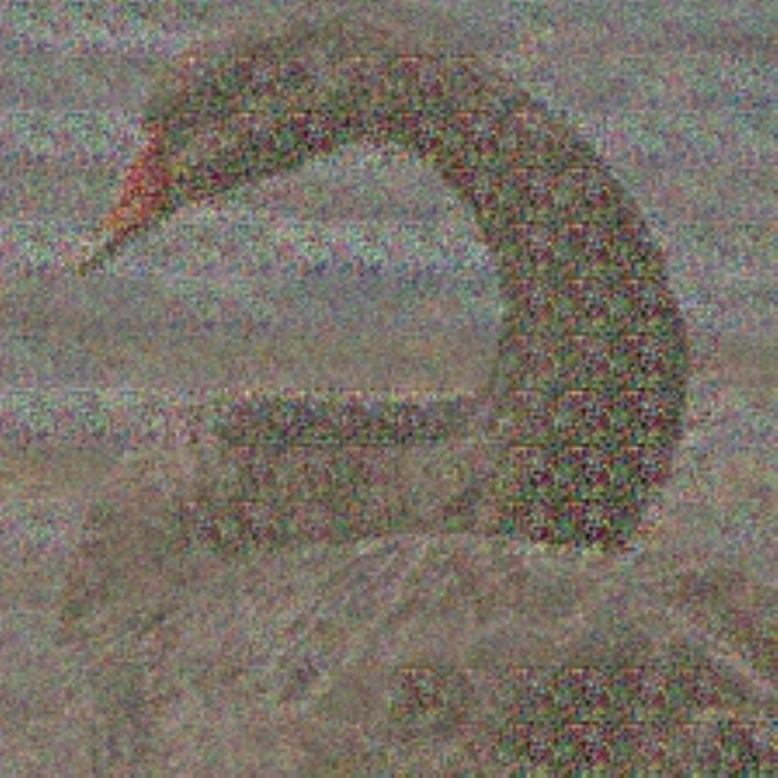}
			\vspace{-8mm}
			\caption*{\footnotesize hidden dimension=192}
		\end{minipage}	
		\hfill
		\begin{minipage}[t]{0.22\linewidth}
			\centering
			\includegraphics[width=\linewidth]{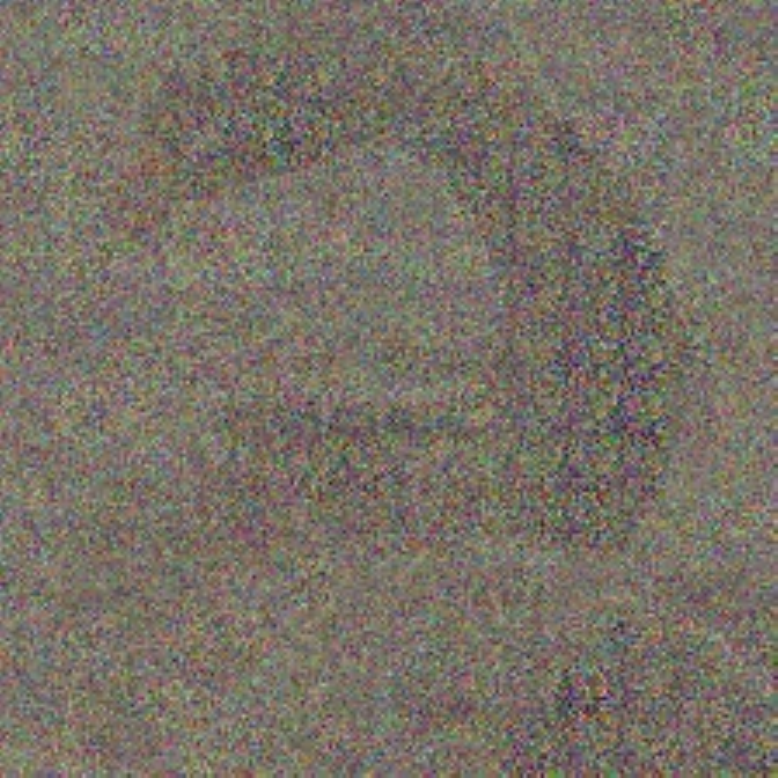}
			\vspace{-8mm}
			\caption*{\footnotesize hidden dimension=96}
		\end{minipage}
		
		\caption{Influences of varying hidden dimension to the reconstruction of APRIL attack.}
		\vspace{0.8cm}
		\label{fig:hidden}
	\end{figure}

	We first testify the influence of changing hidden channel dimensions. A successful closed-form reconstruction relies on the linear system with $p \cdot c$ unknowns and $c \cdot c$ constraints, to be overdetermined. As common configuration suggests  $c$  far larger than $p$,  we deem the linear system to be solvable. To test the robustness of APRIL under different architecture settings, we try four different hidden dimensions. As \cref{fig:hidden} shows, using the original configuration of ViT-base~\cite{vit} cannot be privacy-preserving, the original input image can be entirely leaked by closed-form APRIL attack. Only by shrinking hidden dimensions to a small value (\eg, half of the patch number) can we have solid protection. However, in this configuration, we doubt the network's capacity to gain high accuracy with such small channel number.

	\begin{figure}[h]
		\centering
		\begin{minipage}[t]{0.22\linewidth}
			\centering
			\includegraphics[width=\linewidth]{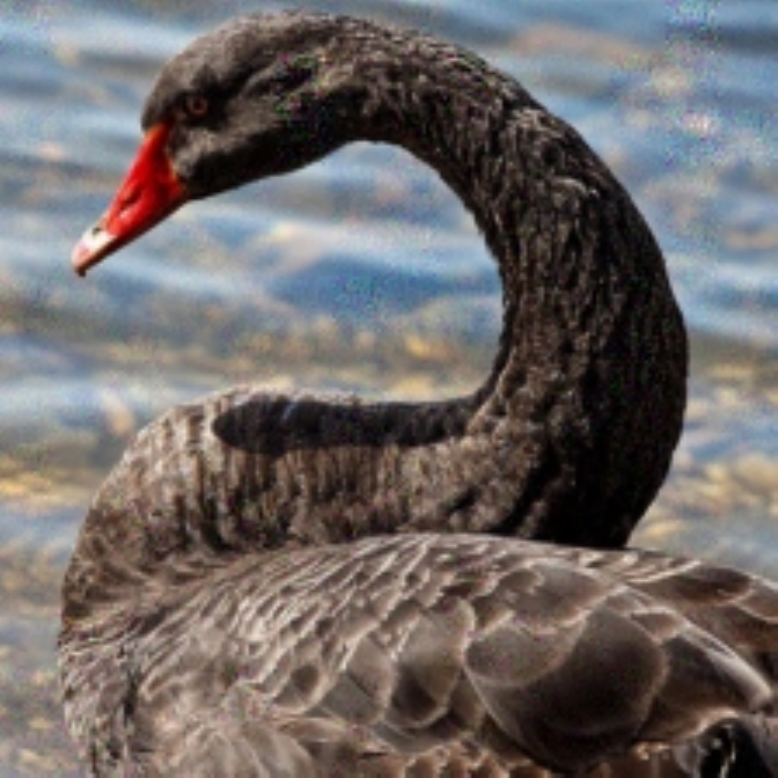}
			\vspace{-8mm}
			\caption*{\footnotesize {Gaussian Var = 0.1x grad norm}}
		\end{minipage}	
		\hfill
		\begin{minipage}[t]{0.22\linewidth}
			\centering
			\includegraphics[width=\linewidth]{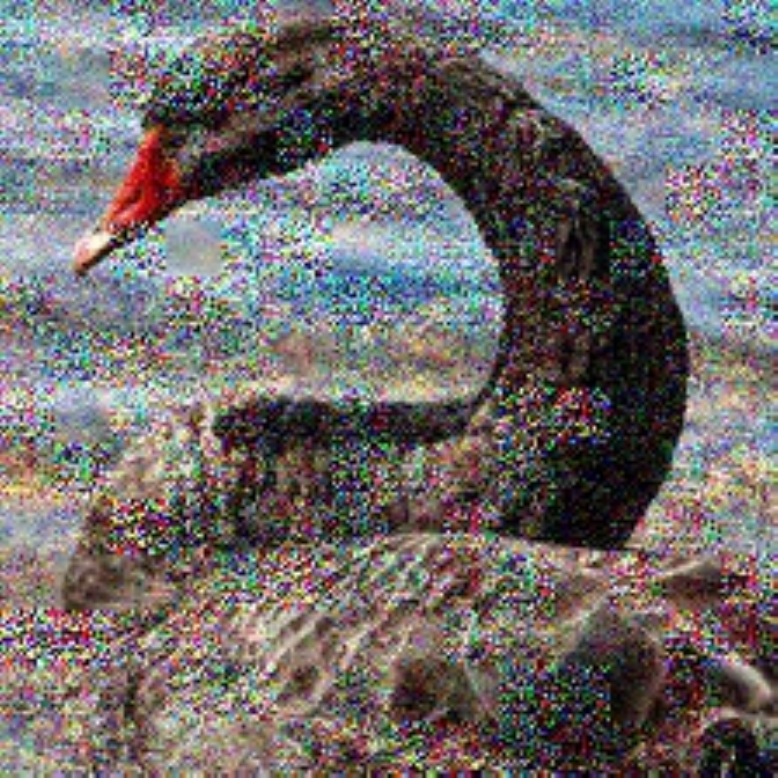}
			\vspace{-8mm}
			\caption*{\footnotesize Gaussian Var = grad norm}
		\end{minipage}
		\hfill
		\begin{minipage}[t]{0.22\linewidth}
			\centering
			\includegraphics[width=\linewidth]{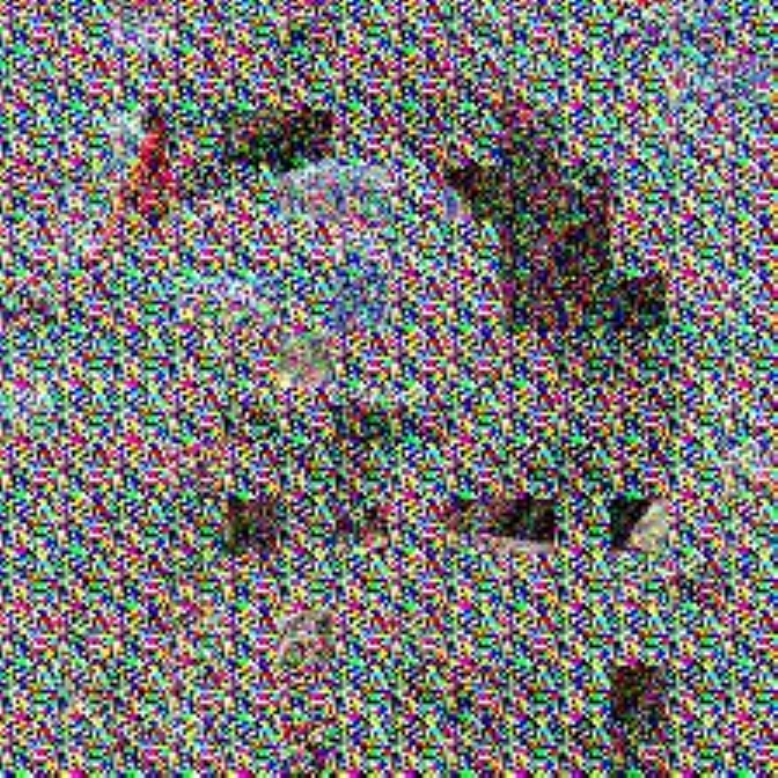}
			\vspace{-8mm}
			\caption*{\footnotesize Gaussian Var = 3x grad norm}
		\end{minipage}	
		\hfill
		\begin{minipage}[t]{0.22\linewidth}
			\centering
			\includegraphics[width=\linewidth]{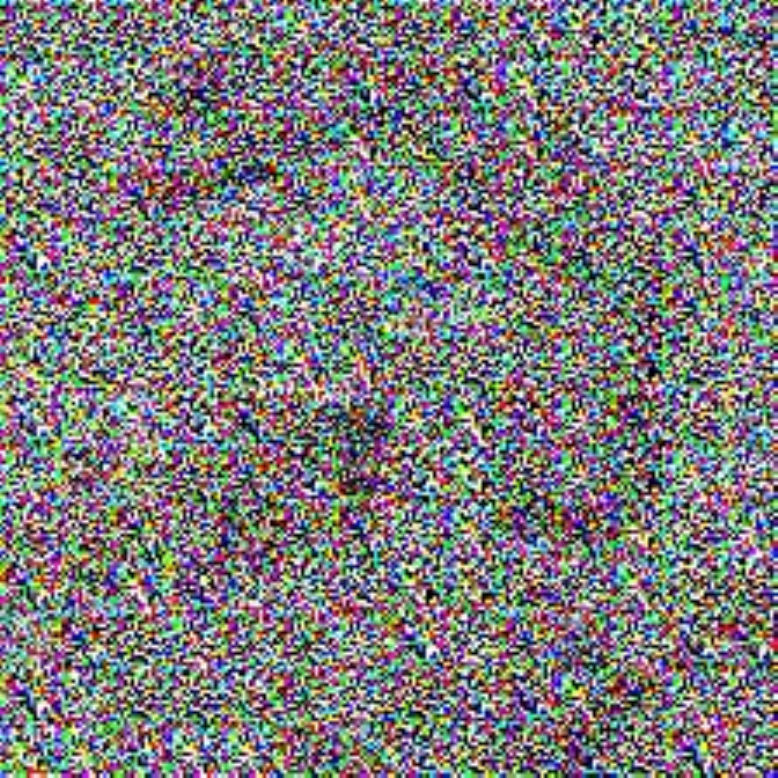}
			\vspace{-8mm}
			\caption*{\footnotesize Gaussian Var = 10x grad norm}
		\end{minipage}
		
		\begin{minipage}[t]{0.22\linewidth}
			\centering
			\includegraphics[width=\linewidth]{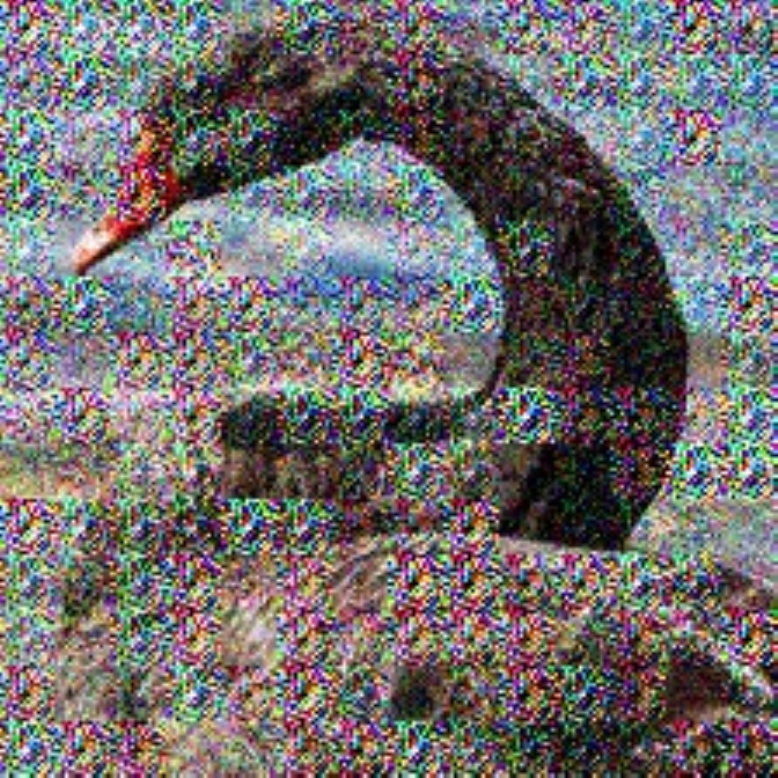}
			\vspace{-8mm}
			\caption*{\footnotesize Laplacian Var = 0.01x grad norm}
		\end{minipage}	
		\hfill
		\begin{minipage}[t]{0.22\linewidth}
			\centering
			\includegraphics[width=\linewidth]{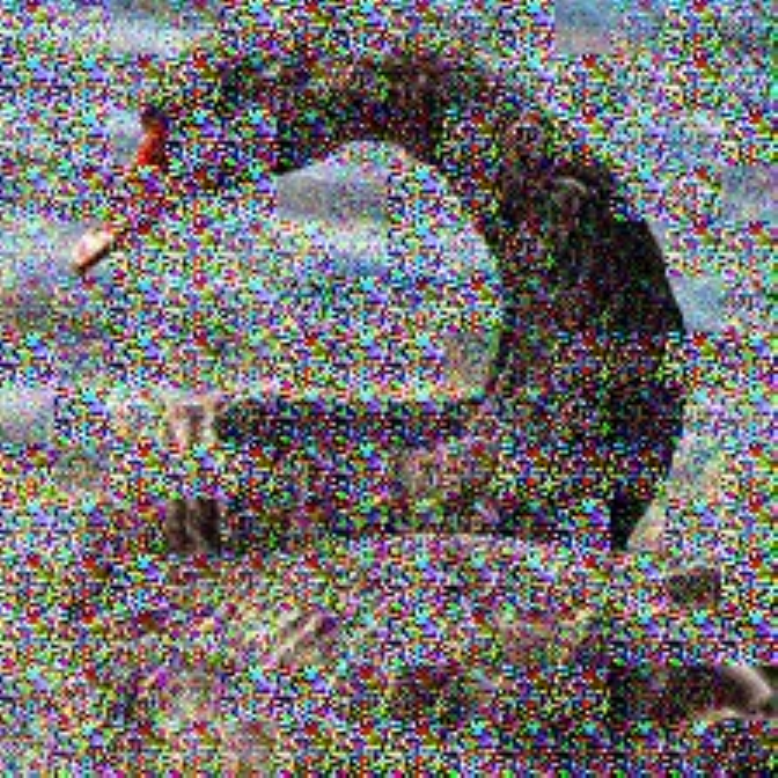}
			\vspace{-8mm}
			\caption*{\footnotesize Laplacian Var = 0.1x grad norm}
		\end{minipage}
		\hfill
		\begin{minipage}[t]{0.22\linewidth}
			\centering
			\includegraphics[width=\linewidth]{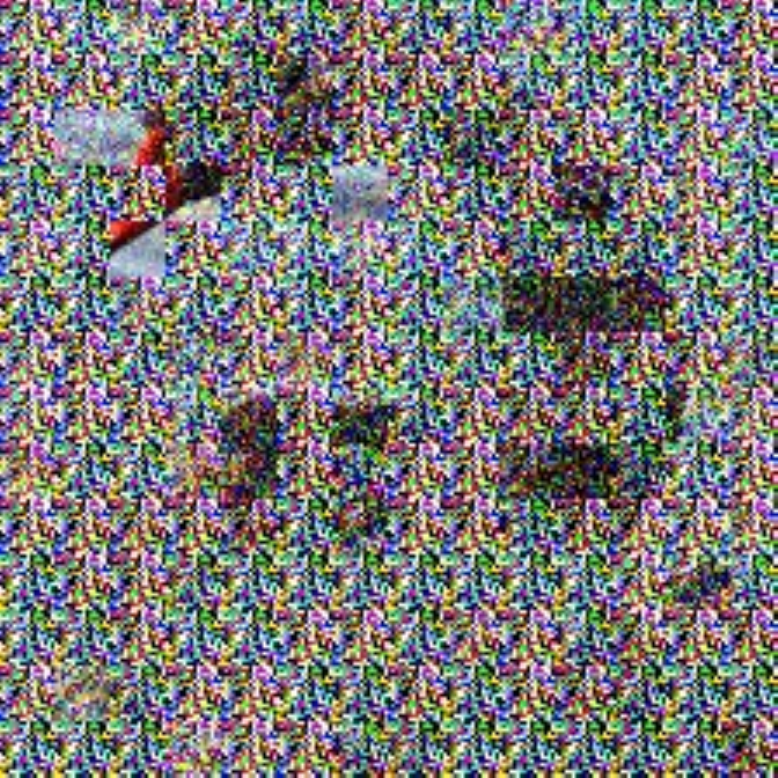}
			\vspace{-8mm}
			\caption*{\footnotesize Laplacian Var =  grad norm}
		\end{minipage}	
		\hfill
		\begin{minipage}[t]{0.22\linewidth}
			\centering
			\includegraphics[width=\linewidth]{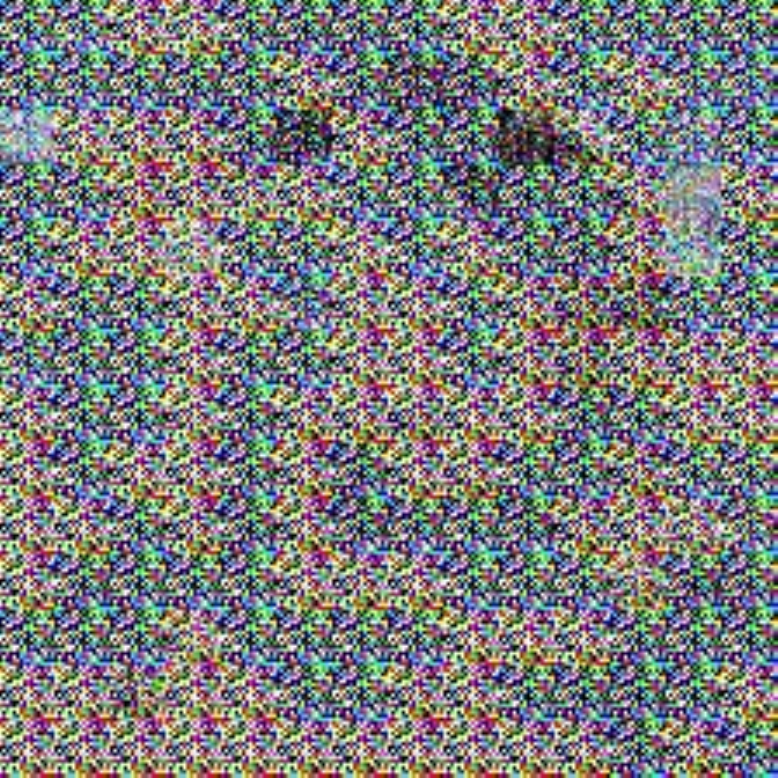}
			\vspace{-8mm}
			\caption*{\footnotesize Laplacian Var = 3x grad norm}
		\end{minipage}	
		\vspace{-2mm}
		\caption{Influences of adding noise to gradients.}
		
		\label{fig:noise}
	\end{figure}

	\begin{figure}
		\centering
		\begin{minipage}[h]{0.85\linewidth}
			\centering
			\includegraphics[width=\linewidth]{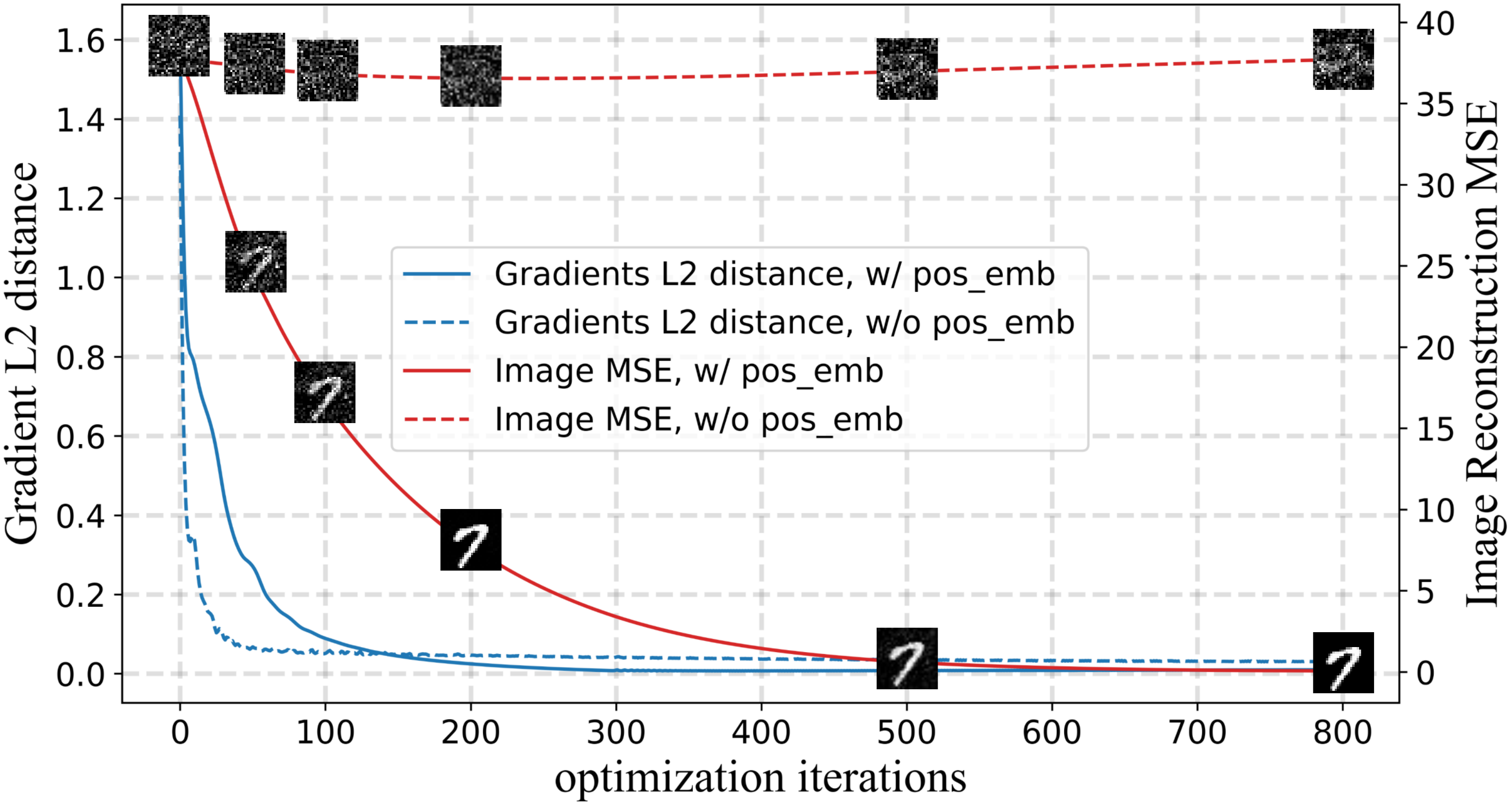}
			\vspace{-7mm}
			\caption*{(A)\ Gradient l2 loss and image MSE on Architecture A}
		\end{minipage}	
		
		\vspace{2cm}
		
		\begin{minipage}[h]{0.85\linewidth}
			\centering
			\includegraphics[width=\linewidth]{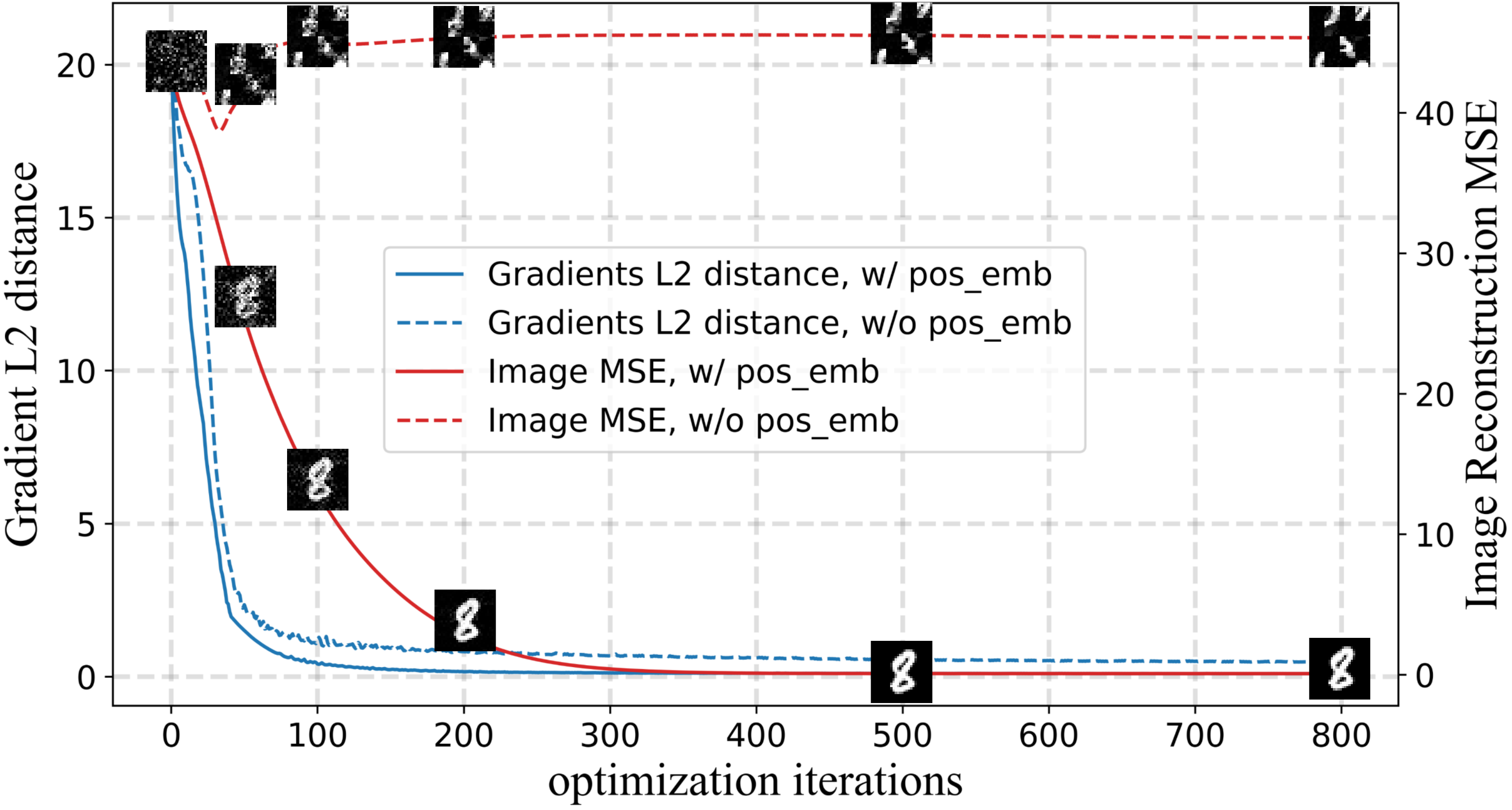}
			\vspace{-7mm}
			\caption*{(B)\ Gradient l2 loss and image MSE on Architecture B}
		\end{minipage}

		\caption{Changes of gradient matching and input reconstruction versus optimization iterations. When position embedding is off, matching gradients does not provide semantically meaningful reconstructions.}
		\vspace{-5mm}
		\label{fig:twinloss}
	\end{figure}
	
	Another more  straightforward way to defend against privacy attacks from gradients is to add noise on gradients. We experiment with Gaussian and Laplacian noises and report results in \cref{fig:noise}. We find that the defense level does not depend on the absolute magnitude of noise variance, but its relative scale to gradient norm. Specifically, when the Gaussian noise variance is lower than 0.1 times (or 0.01 for Laplacian) of gradient norm, the defense won't work. As the variance goes up, the defense ability is greatly promoted. 
	
	\noindent
	\textbf{A Practical and Cheap Defense Scheme.} 
	Apart from adding noise and changing channel dimensions, a more straightforward  way of defending against APRIL is to switch learnable position embedding to a fixed one. In this part, we show that this is a realistic and practical defense, not only for the proposed APRIL, but for all kinds of attacks.
	
	By using a fixed position embedding, clients do not share the gradients \wrt the input. Therefore, it is impossible to perform closed-form APRIL attack. How will optimization-based privacy attacks act when the position embedding is transparent to the attacker?
	
	\begin{wrapfigure}{r}{0.5\linewidth}
		\centering
		\vspace{-1cm}
		\includegraphics[width=0.9\linewidth]{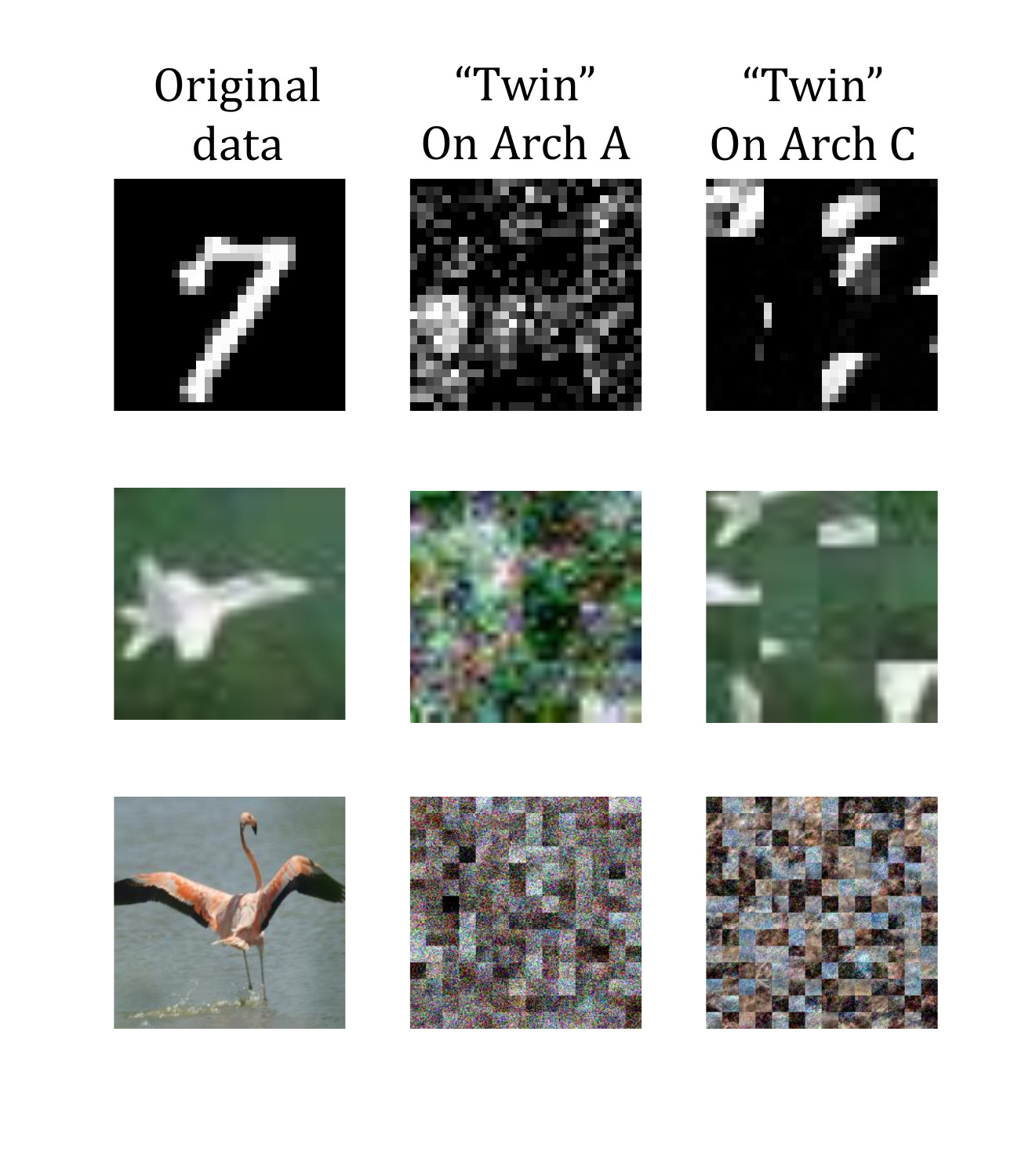}
		\label{fig:dual}
		\vspace{-1cm}
		\caption{Twin data emerge from privacy attack after we stop sharing position embedding. It attested the validity of the defense, in which way confirms that position embedding is indeed the most critical part to Transformer's privacy.}
		\vspace{-1cm}
	\end{wrapfigure}

	We experiment to find out the answer. Note that when position embedding is unknown to the attacker, the optimization-based APRIL attack turns into a more general DLG attack. From results, we notice that similar to \textit{twin data} mentioned by \cite{r-gap}, closing the position embedding gradients seems to result in a family of anamorphic data, which is highly different from original data, but can trigger exactly similar gradients in  a  Transformer. We visualize these patterns as shown in Fig.\red{8}. It's safe to conclude that once we cease sharing position embedding gradients, the gradient matching process will produce semantically meaningless reconstructions. In this way, the attacks fail to break privacy.
	
	To sum up, changing the learnable position to fixed ones or simply not sharing position embedding gradient is practical to prevent privacy leakage in Transformers, which preserves privacy in a highly economic way.
	
	\section{Discussion and Conclusion}
	\label{sec:conclusion}
	In this paper, we introduce a novel approach \textbf{A}ttention \textbf{PRI}vacy \textbf{L}eakage attack (\textbf{APRIL}) to steal private local training data from  shared gradients of a Transformer. The attack builds its success on a key finding that learnable position embedding is the weak spot for Transformer's privacy. Our experiments show that in certain cases the adversary can apply a closed-form attack to directly obtain the input. For broader scenarios, the attacker can make good use of position embedding to perform an optimization-based attack to easily reveal the input. This sets a great challenge to training Transformer models in distributed learning systems. We further discussed possible defenses towards APRIL attack, and verified the effectiveness of using a fixed position embedding. We hope this work would shed light on privacy-preserving network architecture design. In summary, our work has a key finding that learnable position embedding is a weak spot to leak privacy, which greatly advances the understanding of privacy leakage problem for Transformers. Based on the finding, we further propose a novel privacy attack APRIL and discuss effective defending schemes.

	\noindent
	\textbf{Limitation.}
	Our proposed APRIL attack is composed of two parts: closed-form attack when the input gradients are exposed and optimization-based attack otherwise. Closed-form APRIL attack is powerful, nonetheless relies on a strong assumption, which makes it limited to use in real-world Transformer designs. On the other hand, optimization-based APRIL attack implicitly solves a non-linear system. Although they all make good use of gradients from position embedding, there seems to be room to explore a more profound relationship between the two attacks.
	
	\noindent
	\textbf{Potential Negative Societal Impact.} 
	We demonstrate the privacy risk of learnable position embedding, as it is largely used as a paradigm in training Transformers. The privacy attack \textbf{APRIL} proposed in this paper could be utilized by the malicious to perform attack towards existing federated learning systems to steal user data.  We put stress on the defense strategy proposed in the paper as well, and urge the importance of designing privacy-safer Transformer blocks.

	\bibliographystyle{plainnat}
	\bibliography{egbib}
	
	\newpage
	\appendix
	
	\maketitle
	\section{Do Twin Data Come From Lacking of Parameters Involved in Gradient Matching?}
	In \cref{subsec:nopos} and Fig.\red{6} we demonstrate and show that \textit{twin data} can emerge from gradient attacks after disabling learnable position embedding. Formally, a \textit{twin data} appears when a privacy attacker performs gradient matching in a Vision Transformer without matching gradients from position embedding. Some may \textbf{doubt that twin data may come from lack of parameters involved in gradient matching}. We give explicit clues that this is not the case.
	
	We demonstrate this by using examples on MNIST. We use Architecture B (as illustrated in \cref{fig:arch}(B)) with encoder depth 4, hidden dimension of 384. So the overall parameters involved in gradient matching w/ or w/o position embedding are set out in the following table:
	\begin{table}[h]
		\centering
		\label{tab:para}
		\scalebox{1.0}{
			\begin{tabular}{|c|c|}
				\hline
				w/ pos\_emb & w/o pos\_emb \\
				\hline
				7,123,210 & 7,116,682 \\
				\hline
			\end{tabular}
		}
		\caption{Parameters involved in gradient matching, using Architecture (B) for MNIST, encoder depth = 4.}
	\end{table}

	\begin{wrapfigure}{R}{0.5\textwidth}
		\centering
		\begin{minipage}[t]{0.15\linewidth}
			\centering
			\includegraphics[width=\linewidth]{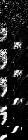}
			\vspace{-5mm}
			\caption*{\small without pos\_emb}
		\end{minipage}	
		\hfill
		\begin{minipage}[t]{0.15\linewidth}
			\centering
			\includegraphics[width=\linewidth]{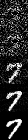}
			\vspace{-5mm}
			\caption*{\small without encoder1}
		\end{minipage}	
		\hfill
		\begin{minipage}[t]{0.15\linewidth}
			\centering
			\includegraphics[width=\linewidth]{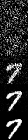}
			\vspace{-5mm}
			\caption*{\small without encoder2}
		\end{minipage}	
		\hfill
		\begin{minipage}[t]{0.15\linewidth}
			\centering
			\includegraphics[width=\linewidth]{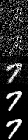}
			\vspace{-5mm}
			\caption*{\small without encoder3}
		\end{minipage}	
		\hfill	
		\begin{minipage}[t]{0.15\linewidth}
			\centering
			\includegraphics[width=\linewidth]{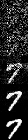}
			\vspace{-5mm}
			\caption*{\small without encoder4}
		\end{minipage}	
		\hfill
		\caption{Results with different sets of parameters involved.}
		\setlength{\belowcaptionskip}{-4cm}
		\vspace{-3cm}
		\label{fig:supp_twinparam}
	\end{wrapfigure}
	
	Here we make a more radical assumption: if we still use position embedding and \textit{twin data} do not occur, when involving much less parameters than in the case without position embedding, then we can empirically prove that "twin data" do not come from lacking of model parameters.
	
	\vspace{5mm}
	For comparisons, we respectively disable the matching of gradients from a specific encoder layer, denoted as \textit{encoder 1}, \textit{encoder 2}, \textit{encoder 3} and \textit{encoder 4}. In these four settings, the overall parameters involved in gradient matching are 6,533,002, much fewer than the case when we disable position embedding, which has 7,123,210 parameters involved.
	
	From \cref{fig:supp_twinparam} we observe that: \textit{twin data} do not come from lacking of parameters involved in optimization.
	
	\newpage
	\section{More Results}
	\subsection{Optimization-based Attack Towards Single Image}
	\noindent
	Here we show more ImageNet attack results in \cref{fig:supp_insingle} Under most cases, the attacks are successful and expose enough information to break privacy. Depending on the content of original images, attacks can have different levels of failures sometimes, when the optimization fall into a bad local minimum and the result has block artifacts. The hardness of attack does not have a preference for certain classes; it depends on the content of original sample. We can observe that, images with higher contrast are easier to have stronger block artifacts on their reconstructions.
	
	We display 6 successful results and 2 failure cases in \cref{fig:supp_insingle}. Each row represents the reconstruction process of a single image; ground-truth input image is shown on the left of each row.

	\subsection{Optimization-based Attack Towards Batched Images}
	Optimization-based APRIL attack obtains great results on batched images. We found that inverting a batch of several images is not markedly harder than inverting single image; with the same setting of iterations and learning rates, even if the batch size is increased from 4 to 24, the reconstructions are still recognizable. We show the results in Fig.\red{12} and Fig.\red{14}. 
	
	\begin{wrapfigure}{R}{0.5\linewidth}
		\centering
		\begin{minipage}{\linewidth}
			\includegraphics[width=0.76\linewidth]{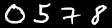}
			\includegraphics[width=0.19\linewidth]{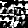}
			\includegraphics[width=0.76\linewidth]{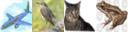}
			\includegraphics[width=0.19\linewidth]{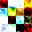}
			\includegraphics[width=0.76\linewidth]{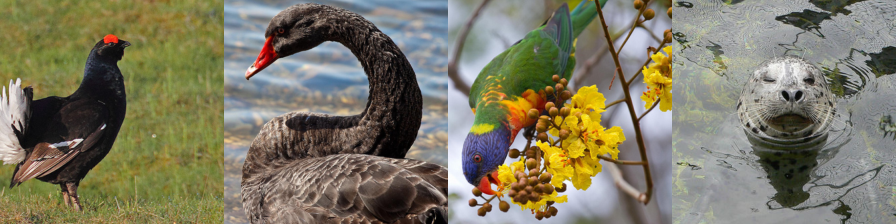}
			\hfill
			\includegraphics[width=0.19\linewidth]{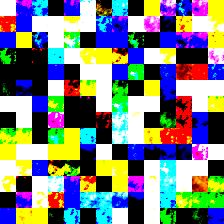}
		\end{minipage}
		
		\caption{Closed-form APRIL attack results on batched inputs, which show perplexing patterns of tanglement.}
		\label{fig:supp_clsbatch}
	\end{wrapfigure}

	\subsection{Closed-Form APRIL Attack Towards Batched Images}
	
	It is worth to notice that the closed-form APRIL attack uses the averaged mean of gradients of all samples in a batch to solve the closed-form solution. Due to the dimensionality, we can only obtain a single image proxy for the whole batch. The results are shown in \cref{fig:supp_clsbatch}, which give almost no useful information towards original input images.

	As another closed-form attack, R-GAP~\cite{r-gap} provides results of their approach on batched data as well in their Figure 7. Their results are easy to understand since they look like additive combinations of the original batch of images.
	
	In contrast, our results are not recognizable and easy to be interpreted since we do not obtain the solution in an additive way, but in a multiplicative way.

	\begin{figure*}
		\centering
		\begin{minipage}[h]{\linewidth}
			\includegraphics[width=0.125\linewidth]{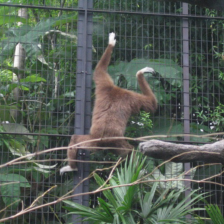}
			\includegraphics[width=0.875\linewidth]{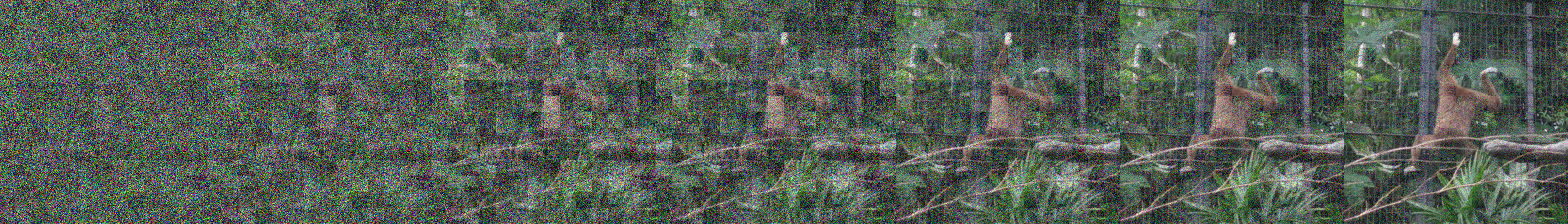}
			\includegraphics[width=0.125\linewidth]{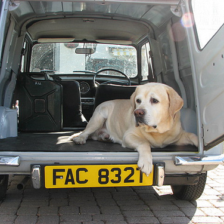}
			\includegraphics[width=0.875\linewidth]{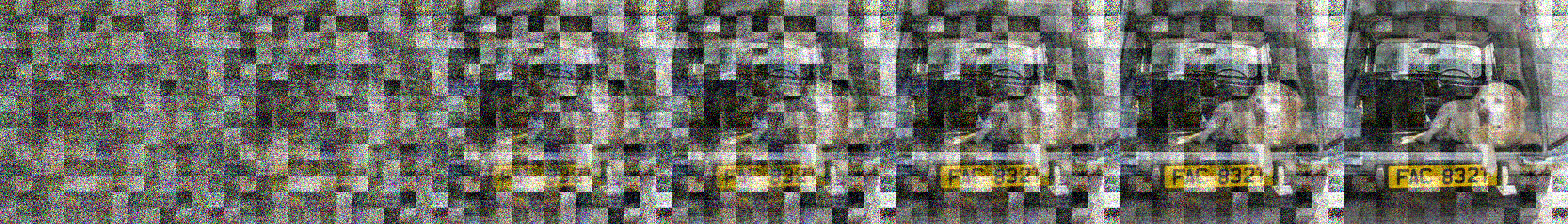}
			\includegraphics[width=0.125\linewidth]{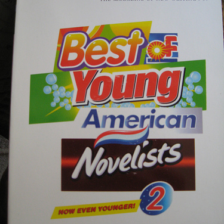}
			\includegraphics[width=0.875\linewidth]{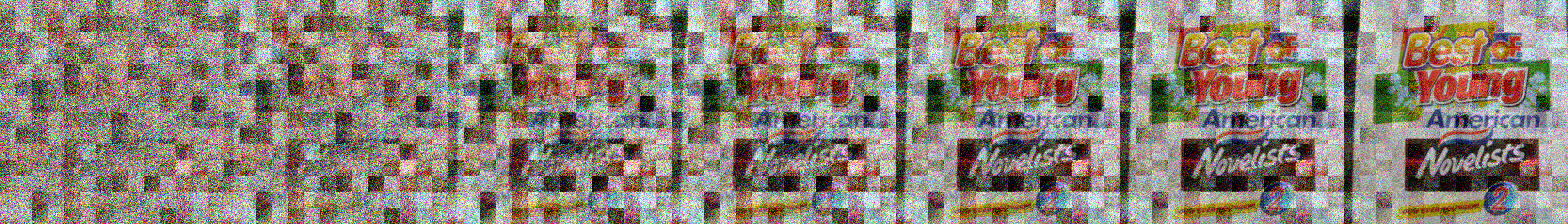}
			\includegraphics[width=0.125\linewidth]{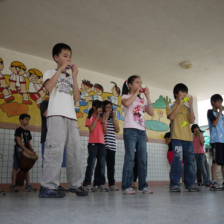}
			\includegraphics[width=0.875\linewidth]{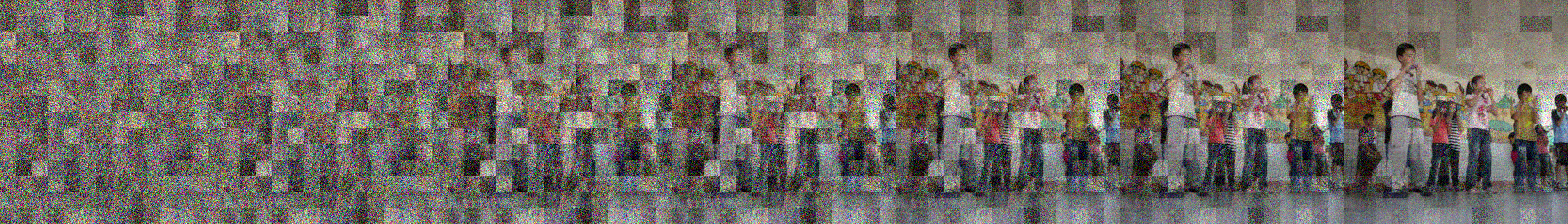}
			\includegraphics[width=0.125\linewidth]{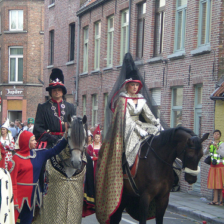}
			\includegraphics[width=0.875\linewidth]{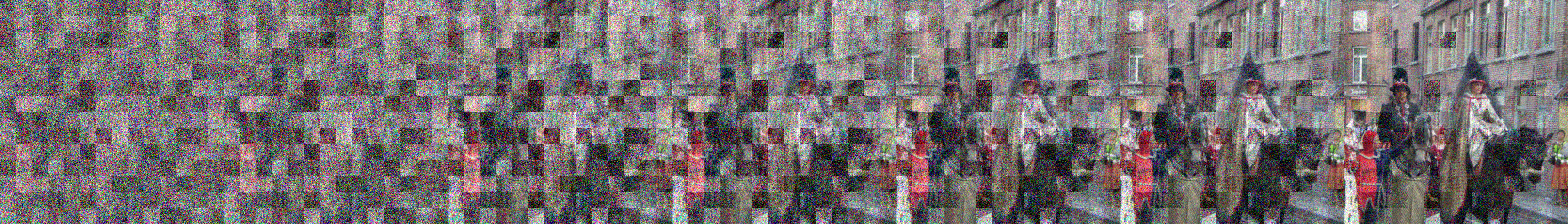}
			\includegraphics[width=0.125\linewidth]{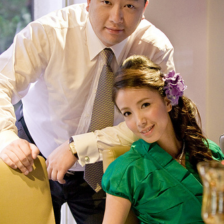}
			\includegraphics[width=0.875\linewidth]{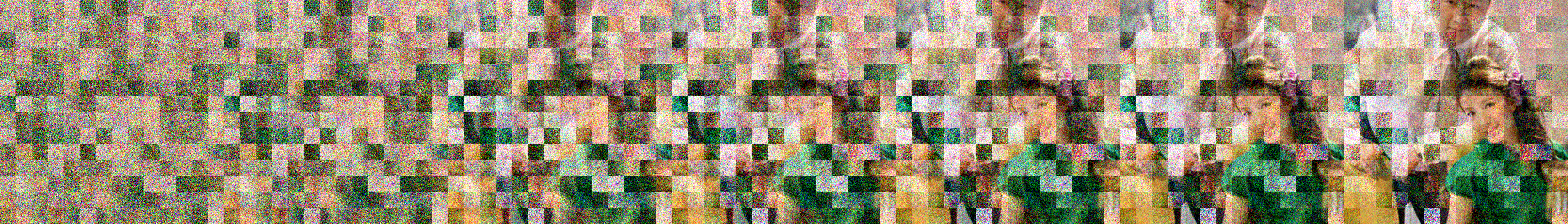}
			\caption*{Successful attack results.}
		\end{minipage}
		
		\begin{minipage}{\linewidth}
			\includegraphics[width=0.125\linewidth]{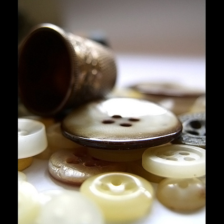}
			\includegraphics[width=0.875\linewidth]{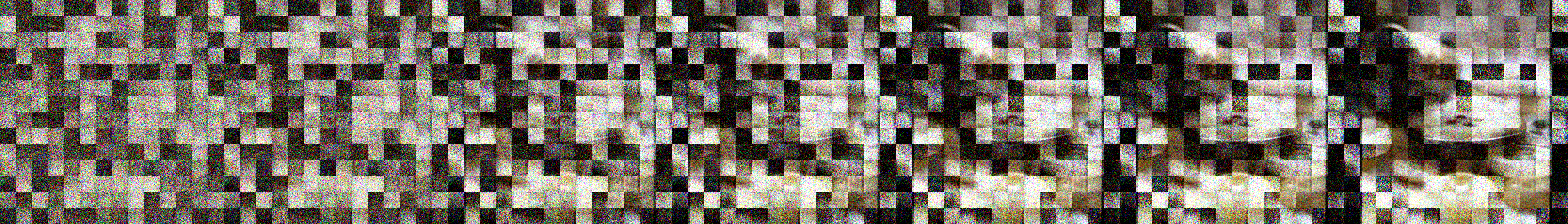}
			\includegraphics[width=0.125\linewidth]{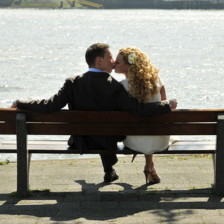}
			\includegraphics[width=0.875\linewidth]{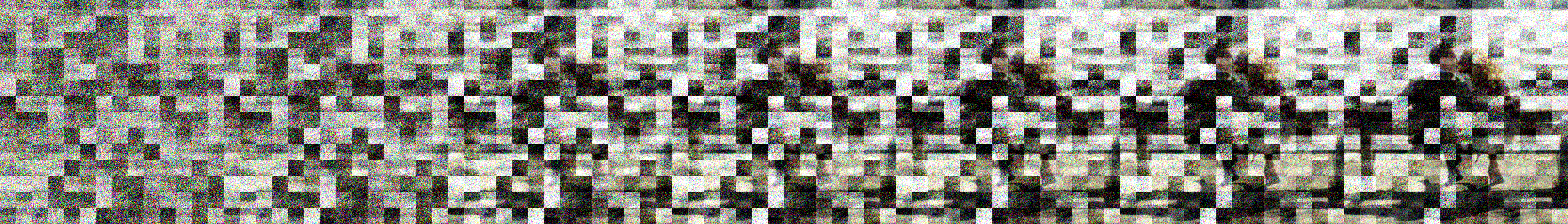}
			\caption*{Failure attack results.}
		\end{minipage}
		
		\caption{More results on ImageNet for single image attacks. From left to right: the ground-truth image; reconstrutions at iteration 50, 100, 500, 1000, 2000, 3000, 5000.}
		\label{fig:supp_insingle}
	\end{figure*}

	\begin{figure*}[t]
		\centering
		\begin{minipage}{\linewidth}
			\centering
			\includegraphics[width=0.5\linewidth]{imagenet_batch4.png}
			\includegraphics[width=0.5\linewidth]{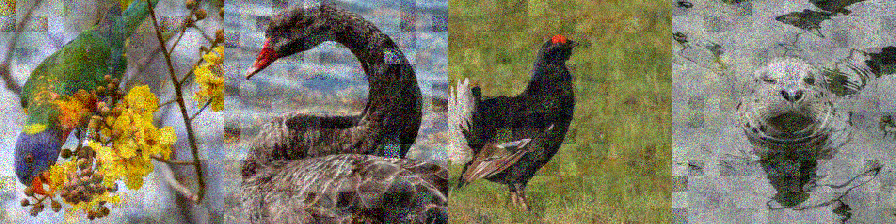}
			\label{supp_4samples}
			\caption{Reconstruction of a batch of 4 ImageNet samples. The first row: ground-truth batch images. The second row: reconstruction results by optimization-based APRIL.}
		\end{minipage}
		\begin{minipage}{\linewidth}
			\centering
			\includegraphics[width=\linewidth]{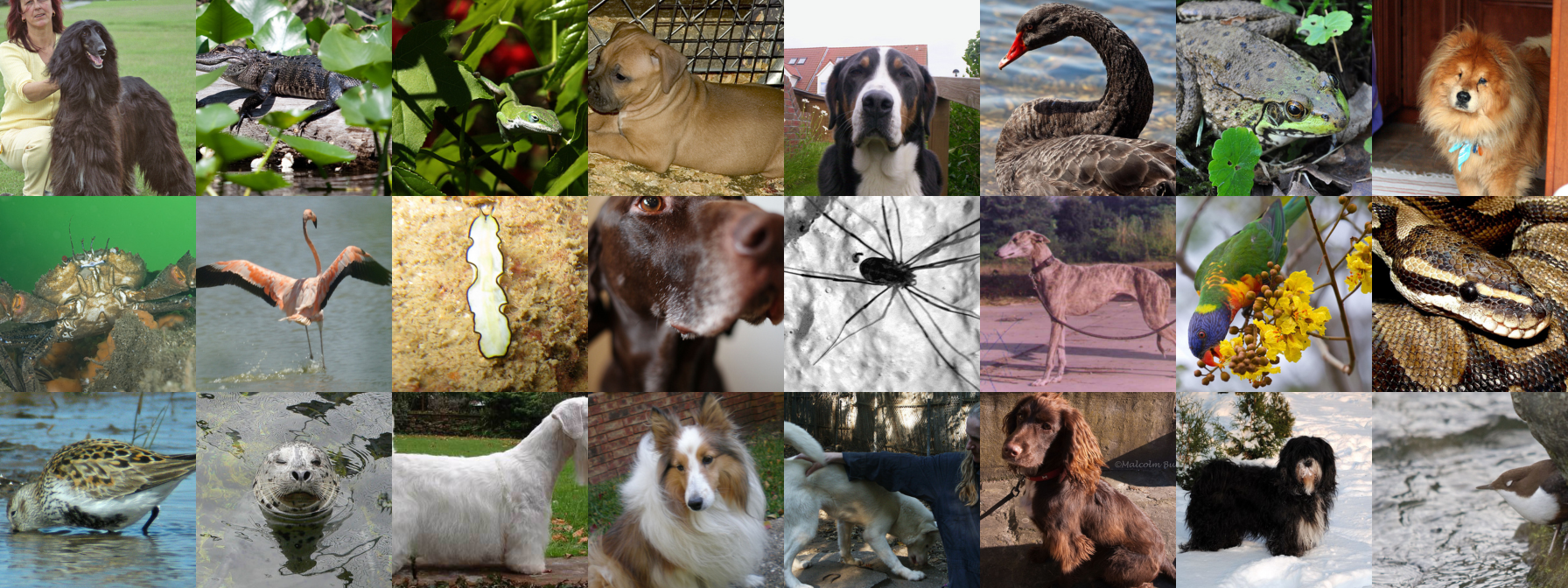}
			\label{supp_24gt}
			\caption{A batch of 24 ImageNet samples.}
			\vfill
			\vspace{7mm}
			\includegraphics[width=\linewidth]{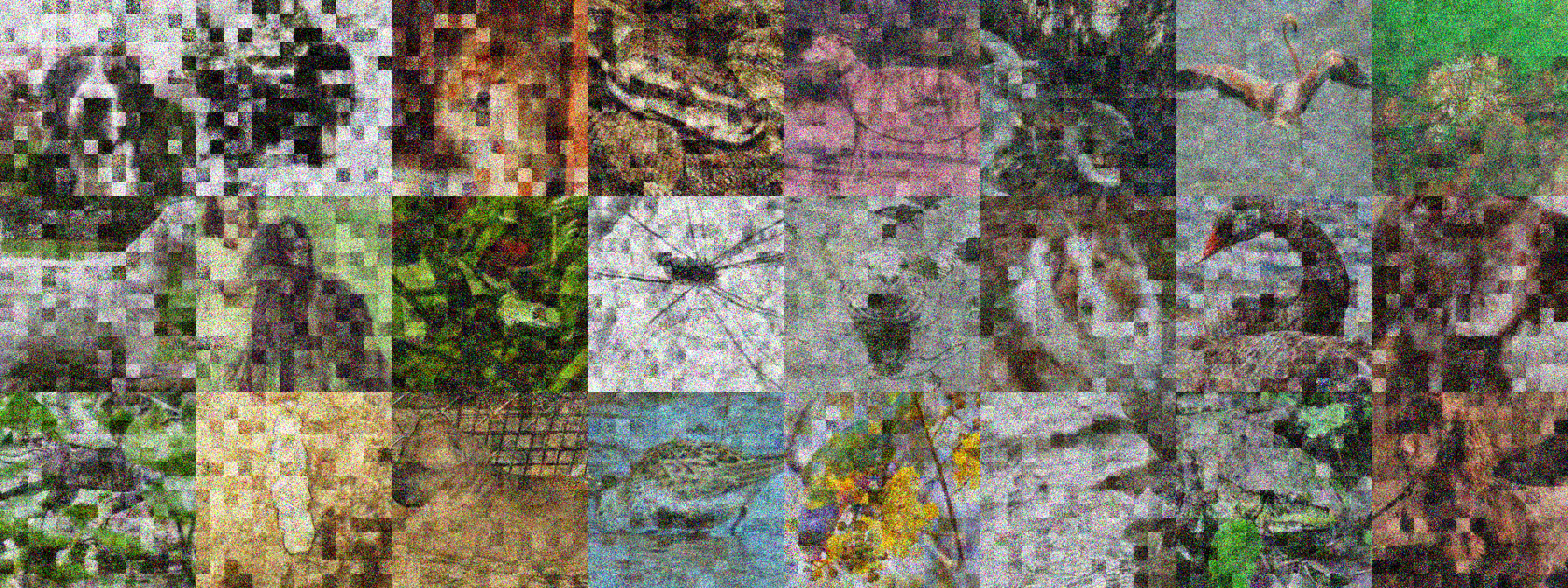}
			\label{supp_24recon}
			\caption{Reconstruction of the image batch in Fig. \red{12} using optimization-based APRIL.}
		\end{minipage}
		\caption*{}
		\label{fig:supp_optbatch}
	\end{figure*}
	
\end{document}